\documentclass{article}

\usepackage[preprint]{neurips_2026}

\usepackage[utf8]{inputenc}
\usepackage[T1]{fontenc}
\usepackage{url}
\usepackage{booktabs}
\usepackage{amsfonts}
\usepackage{amsmath}
\usepackage{amssymb}
\usepackage{amsthm}
\usepackage{nicefrac}
\usepackage{microtype}
\usepackage{xcolor}
\usepackage{graphicx}
\usepackage{wrapfig}

\makeatletter
{\catcode`\$=3 \gdef\ordmathshift{$}}
\AtBeginDocument{\catcode`\$=\active}
{\catcode`\$=\active
 \protected\gdef${\ifmmode\ordmathshift\endgroup
   \else\begingroup\ifdim\f@size\p@=10\p@\small\fi\ordmathshift\fi}}
\makeatother

\newtheoremstyle{boldital}%
  {\topsep}{\topsep}
  {\itshape}
  {0pt}
  {\bfseries}
  {.}
  { }
  {\thmname{#1}\thmnumber{ #2}\thmnote{ (\itshape #3)}}

\theoremstyle{boldital}
\newtheorem{theorem}{Theorem}
\newtheorem{proposition}{Proposition}
\newtheorem{definition}{Definition}
\newtheorem{assumption}{Assumption}
\newtheorem{lemma}{Lemma}

\newcommand{\R}{\mathbb{R}}

\usepackage[breaklinks=true,hypertexnames=false]{hyperref}

\title{Attributing Emergence in Million-Agent Systems}

\author{%
  \textbf{Ling Tang}\textsuperscript{1,2}\thanks{Equal contribution. Work done during an internship at Shanghai Artificial Intelligence Laboratory.}\quad
  \textbf{Jilin Mei}\textsuperscript{1,3}\footnotemark[\value{footnote}]\quad
  \textbf{Qian Chen}\textsuperscript{1,4}\quad
  \textbf{Qihan Ren}\textsuperscript{1,2}\quad
  \textbf{Linfeng Zhang}\textsuperscript{2}\\
  \textbf{Quanshi Zhang}\textsuperscript{2}\quad
  \textbf{Jing Shao}\textsuperscript{1}\quad
  \textbf{Xia Hu}\textsuperscript{1}\quad
  \textbf{Dongrui Liu}\textsuperscript{1}\thanks{Correspondence
to: Dongrui Liu <liudongrui@pjlab.org.cn>}\\[4pt]
  \textsuperscript{1}Shanghai Artificial Intelligence Laboratory \quad
  \textsuperscript{2}Shanghai Jiao Tong University \\
  \textsuperscript{3}Fudan University \quad
  \textsuperscript{4}Tongji University}

\begin{document}

\maketitle

\begin{abstract}
Large language models (LLMs) can simulate human-like reasoning and decision-making in individual agents. LLM-powered multi-agent systems (MAS) combine such agents to simulate population-scale social phenomena such as polarization, information cascades, and market panics. Such studies require attributing macro emergence to individual agents, but existing axiomatic methods scale combinatorially in $N$ and have been confined to $N \lesssim 10^3$, while the phenomena they explain occur at $N \geq 10^6$. We address this gap by adapting Aumann--Shapley path-integral attribution to LLM-powered MAS at million-agent scale; the resulting method satisfies all four axioms, runs three to five orders of magnitude faster than sampled Shapley on the same hardware, and extends feasible axiomatic attribution by over three orders of magnitude (a $1670\times$ jump). We use this method to test the scale gap empirically: across 14 days of public Bluesky data ($1{,}671{,}587$ active users, five topics), we compute the attribution at both full scale and the visibility-biased $N = 10^2$ convenience sample used by small-scale studies, and the two disagree structurally. At full scale the long tail and middle tier jointly carry the majority; the biased small panel shifts about twice that share onto the upper follower tiers ($48\%$ versus $24\%$). We then prove that the disagreement cannot in general be reduced by post-hoc rescaling: an Attribution Scaling Bias theorem shows that a reconciling global rescaling factor exists exactly when the macro indicator is linear over agents, and our nonlinear indicators give residuals of $0.10$--$0.98$. For such nonlinear indicators, full-scale attribution is therefore a requirement rather than a methodological choice.
\end{abstract}

\section{Introduction}
\label{sec:intro}

\begin{figure}[!t]
\centering
\includegraphics[width=\linewidth]{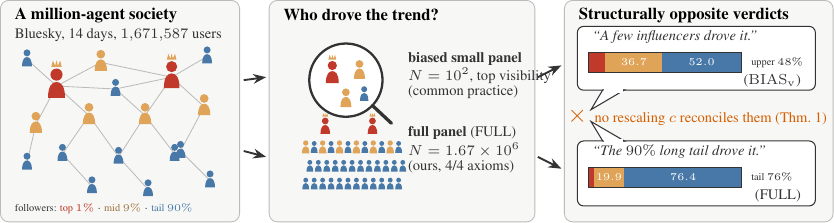}
\caption{\textbf{Attribution gives structurally opposite verdicts on a biased small panel and on the full system.} Aumann--Shapley path-integral attribution (4/4 axioms, $O(N)$ time) on the visibility-biased $N = 10^2$ panel ($\rm BIAS_v$) and on the full $N = 1.67 \times 10^6$ Bluesky panel (Mythos, $f^{\rm heat}$). The biased panel assigns $48\%$ to the combined upper tiers ($G_{\rm top} + G_{\rm mid}$) versus $24\%$ at full scale, where the $90\%$ long tail carries $76.4\%$; no global rescaling factor reconciles the two (residual $\varepsilon \approx 0.13$; Theorem~\ref{thm:asb}).}
\label{fig:teaser}
\end{figure}

Large language models (LLMs) can simulate human-like reasoning and decision-making in individual agents~\citep{argyle2023out, horton2023large}. Multi-agent systems (MAS) have long been used to study how interactions among autonomous agents produce macro-level emergent phenomena~\citep{wooldridge2009introduction}. LLM-powered MAS combine the two, retaining the social realism of LLM agents at the population scale of classical MAS, and have become a common platform for simulating large-scale social phenomena. Recent applications span generative agent societies~\citep{park2023generative}, macroeconomic activity~\citep{li2024econagent}, financial markets~\citep{yang2025twinmarket}, and million-agent social platforms~\citep{yang2024oasis, piao2025agentsociety, ashery2025emergent}.

These systems reproduce macro emergent phenomena observed in human society, including polarization~\citep{bail2018exposure}, information cascades~\citep{vosoughi2018spread, goel2016structural}, and market panics~\citep{brunnermeier2009deciphering}. Scientific use of these reproductions raises two questions about each phenomenon: \emph{which agents drove it?} and \emph{when did it turn?} The macro curve answers neither; both require decomposing the macro indicator into per-agent, per-step contributions. Throughout, ``agent $i$ drove the phenomenon'' is shorthand for the structural statement that agent $i$ carries a large share of the axiomatic attribution of the macro indicator, not a claim of counterfactual causation.

This paper studies emergence attribution in LLM-powered MAS at realistic scale. Existing attribution methods require re-running the simulation over many agent subsets, at cost combinatorial in the number of agents $N$. Recent adaptations of axiomatic Shapley attribution to LLM-powered MAS~\citep{tang2026interpreting, luagent} remain confined to $N \lesssim 10^3$, while the social phenomena they target, such as online cascades~\citep{vosoughi2018spread} and election-period opinion turns~\citep{bond2012million}, occur at $N \geq 10^6$. Whether agent-level conclusions drawn at small scale survive at realistic scale is an open question. We give the first systematic answer.

We develop an attribution framework that satisfies the four Aumann--Shapley axioms and scales to millions of agents (Section~\ref{sec:theory}). The framework rests on the Aumann--Shapley path integral~\citep{aumann1974values}, reaches $N = 10^6$ in under ten milliseconds, and is three to five orders of magnitude faster than sampled Shapley on the same hardware, over three orders of magnitude (a $1670\times$ jump) beyond the largest agent count reported in prior axiomatic work.

Equipped with this framework, we run a controlled cross-scale study on the complete public Bluesky dataset~\citep{failla2024ibluesky, seckin2025rise} of $1{,}671{,}587$ active users over April 7--20, 2026 (Section~\ref{sec:finding}). For each topic we attribute the macro indicator both on the full panel and on the visibility-biased convenience sample of $N = 10^2$ that an LLM-powered MAS case study would in practice use. The two answers disagree structurally: at full scale the long tail and middle tier jointly carry the majority of the attribution; under the biased small panel the same axiomatic method shifts about twice that share onto the upper follower tiers ($48\%$ versus $24\%$; Figure~\ref{fig:teaser}, Table~\ref{tab:flip-main}). A uniform random sample of the same size produces no such inflation, though it undershoots rather than recovers the full-scale shares: no small panel of either composition reproduces the full-scale verdict (Appendix~\ref{app:numbers}).

Nor can the disagreement be removed by post-processing. Our \emph{Attribution Scaling Bias} theorem (Theorem~\ref{thm:asb}, Section~\ref{sec:theory}) shows that exactly one reconciling global factor always works under linearity, while outside the linear class the guarantee fails: the measured residual is uniformly positive for every nonlinear indicator we study, with an exact formula for variance (Table~\ref{tab:theorem-test}, Proposition~\ref{prop:residual}). The macro indicators of practical interest, including variance, the Gini coefficient, cascade depth~\citep{goel2016structural}, and learnable nonlinear value networks, are all nonlinear, so post-hoc correction cannot be relied on.

We do not claim that every LLM-powered MAS study must be conducted at million-agent scale. We claim only that any conclusion of the form ``\emph{this kind of agent drives this kind of phenomenon}'' should rest either on analysis at the phenomenon's realistic scale or on cross-scale consistency evidence.

\section{Related Work}
\label{sec:related}

Prior research on LLM-powered MAS has not engaged directly with the scale dependence we report in Section~\ref{sec:finding} (agent count rather than model parameters; Appendix~\ref{app:scaling}). We discuss the three lines of work that bear most directly on agent-level attribution below.

\textbf{Qualitative reproduction of macro emergence in LLM-powered MAS.} A first line of work has demonstrated that LLM-powered MAS can reproduce a wide range of social phenomena. Generative agent societies~\citep{park2023generative}, macroeconomic activity~\citep{li2024econagent}, financial markets~\citep{yang2025twinmarket}, and large-scale social platforms~\citep{yang2024oasis, piao2025agentsociety, gao2023s3, ashery2025emergent, tang2024gensim} have replicated polarization, information cascades, market crashes, and convention formation. Parallel work on collaborative agent pipelines such as AutoGen~\citep{wu2023autogen}, MetaGPT~\citep{hong2024metagpt}, ChatDev~\citep{qian2024chatdev}, and CAMEL~\citep{li2023camel} has reported emergent coordination in task-solving settings, and recent surveys~\citep{mou2024survey} have consolidated these directions. We do not add a new emergent phenomenon; we ask who drove a known one, and when.

\textbf{Axiomatic Shapley-type attribution.} A second line has applied cooperative-game attribution to multi-agent settings, building on the \citet{shapley1953value} value and its non-atomic extension by \citet{aumann1974values}. The Shapley value has been adapted to multi-agent reinforcement learning for credit assignment~\citep{wang2020shapleyq, wang2022shaq}, to data valuation in machine learning~\citep{ghorbani2019dataShapley, jia2019knn, wang2023banzhaf}, and most recently to LLM-powered MAS itself~\citep{tang2026interpreting, luagent}. These methods inherit strong axiomatic guarantees but require evaluating the system on many agent subsets, with cost combinatorial in $N$; published experiments accordingly remain at $N \lesssim 10^3$. Leave-one-out and Banzhaf semi-values~\citep{koh2017understanding, wang2023banzhaf, luagent} satisfy three of the four axioms (symmetry, dummy, linearity) while losing efficiency, and share the same rerun cost. We retain the same axioms while removing the rerun cost. The computational core we use is not new in itself: the straight-path Aumann--Shapley value coincides with Integrated Gradients~\citep{sundararajan2017axiomatic} applied with agents as features. What is new here is the agent-as-feature MAS framing, the closed-form $O(N)$ instantiations that reach $N = 10^6$, and the cross-scale analysis this unlocks (Theorem~\ref{thm:asb}).

\textbf{Heuristic and prompt-based tracing of MAS failures.} A third line has framed agent-level attribution as a labelling problem over execution traces. MAST~\citep{cemri2025mast}, Who\&When~\citep{zhang2025whoandwhen}, AgenTracer~\citep{zhang2025agenttracer}, and FAMAS~\citep{ge2025famas} have used trained judges, spectrum-based suspicion scores, or LLM-as-judge prompts to localize faulty agents in failed multi-agent runs. These methods scale to large traces but provide no axiomatic guarantees, and reported agent-level accuracy has remained below $55\%$ on the Who\&When benchmark~\citep{zhang2025whoandwhen}. We target macro emergence rather than failure, and rely on axioms rather than learned judges.

Table~\ref{tab:methods} summarizes the three lines along the two axes that matter for our use case: how many of the four Aumann--Shapley axioms (efficiency, symmetry, dummy, linearity) the method satisfies, and the largest agent count at which the method has been reported in published work. The Aumann--Shapley attribution we develop in Section~\ref{sec:theory} retains all four axioms while extending the reported $N$ by over three orders of magnitude (a $1670\times$ jump).

\begin{table}[!t]
\caption{Attribution methods compared on two axes: number of axioms satisfied (efficiency, symmetry, dummy, linearity) and the largest agent count reported in published experiments.}
\label{tab:methods}
\centering
\small
\setlength{\tabcolsep}{8pt}
\begin{tabular}{lcc}
\toprule
Method & Axioms & Reported $N$ \\
\midrule
MAST~\citep{cemri2025mast}                & $0/4$  & $\leq 5$        \\
Who\&When~\citep{zhang2025whoandwhen}     & $0/4$  & $\leq 5$        \\
AgenTracer~\citep{zhang2025agenttracer}   & $0/4$  & $\leq 5$        \\
FAMAS~\citep{ge2025famas}                 & $0/4$  & $\leq 5$        \\
Leave-one-out~\citep{luagent}             & $3/4$  & $\sim 10^3$    \\
Banzhaf~\citep{luagent}                   & $3/4$  & $\sim 10^3$    \\
Shapley~\citep{tang2026interpreting}      & $4/4$  & $\sim 10^3$    \\
\textbf{Aumann--Shapley (ours)}            & $\mathbf{4/4}$  & $\boldsymbol{1.67 \times 10^6}$ \\
\bottomrule
\end{tabular}
\end{table}

\section{Finding: Attribution Flips with Sampling Scale}
\label{sec:finding}

This section reports the central empirical observation, using complete public Bluesky data as the testbed: it carries the structural features that LLM-powered MAS studies aim to reproduce (heavy-tailed follower distribution, cross-tier engagement dynamics, real event timing) at the $N \geq 10^6$ scale where the phenomena of interest occur. The method itself requires only per-agent features over time and applies to any LLM-powered MAS pipeline whose macro indicator is (or is smoothed into) a $C^2$ function of those features (Section~\ref{sec:theory}); on three LLM-powered MAS scenarios it agrees with sampled Shapley in agent ranking wherever the latter is feasible, and remains computable where it is not (Appendix~\ref{app:llmmas}).

At full scale the long tail carries the majority of the attribution; on the small visibility-biased convenience sample that an LLM-powered MAS case study would in practice use, the same axiomatic method concentrates the result on the upper follower tiers, across four value functions and five topics spanning technology, politics, sports, society, and entertainment (Sections~\ref{sec:finding-setup}--\ref{sec:finding-flip}, Table~\ref{tab:flip-main}). Full attribution also recovers topic-specific temporal structure that the biased panel collapses to a single shape (Section~\ref{sec:finding-universal}).

\subsection{Setup}
\label{sec:finding-setup}

\textbf{Data.} We use the complete public AT-Protocol Jetstream feed of Bluesky from April 7 to April 20, 2026, fourteen days.\footnote{Data ethics, public-data boundaries, de-identification safeguards, and misuse-risk considerations are in Appendix~\ref{app:ethics}; the firehose endpoint, exact capture window, filtering pipeline, topic matching, and sampling seeds for end-to-end reproducibility are in Appendix~\ref{app:reproducibility}.} After standard filtering (Appendix~\ref{app:topics}) the active panel contains $|[N]| = 1{,}671{,}587$ users. We select five topics in distinct domains and report all five in Table~\ref{tab:flip-main}; the figures focus on the technology-and-subculture topic Mythos, whose multi-modal temporal trajectory makes the per-day cross-tier structure most visible (other topics: Appendix~\ref{app:topics}).

\textbf{Per-agent feature.} Each agent $i$ is described by a three-dimensional feature $z_i = (a_i, b_i, c_i)$, where $a_i = \log(1+\text{followers}_i)$ is reach, $b_i = \log(1+\text{topic\_posts}_i)$ is topic-specific activity, and $c_i = \log(1+\text{topic\_replies\_received}_i)$ is topic-specific resonance. For any subset $S \subseteq [N]$ with $|S| = n$, we write $m_d(S) := (1/n)\sum_{i \in S} d_i$ for the per-coordinate mean, $d \in \{a,b,c\}$, and $g_i := a_i + b_i + c_i$ for the per-agent additive composite. Inactive accounts retain a non-zero $a_i$, so high-follower users without topic engagement can still receive non-zero attribution.

\textbf{Three-tier follower partition.} We rank all $|[N]|$ users by follower count and split them once into three disjoint groups: $G_{\rm top}$, the top $1\%$ ($16{,}716$ users); $G_{\rm mid}$, the next $9\%$ ($150{,}443$ users); and $G_{\rm tail}$, the remaining $90\%$ ($1{,}504{,}428$ users). Two largely disjoint alternative elite anchors (top $1\%$ by posts and by replies received) are analyzed in Appendix~\ref{app:kol-anchors}.

\textbf{Relative attribution share.} For any subset $S \subseteq [N]$ and any group $G$, we report the share of the within-$S$ normalized attribution carried by agents that lie in $G$:
\begin{equation}
R_G^S \;:=\; \sum_{i \in S \cap G} \tilde\phi_i^S,
\label{eq:R-def}
\end{equation}
where $\tilde\phi_i^S$ is the normalized Aumann--Shapley attribution of agent $i \in S$ (formally Definition~\ref{def:as} in Section~\ref{sec:theory}, with $\sum_{i \in S} \tilde\phi_i^S = 1$ in the cases reported here). Because the three follower tiers partition the panel, $R_{G_{\rm top}}^S + R_{G_{\rm mid}}^S + R_{G_{\rm tail}}^S = 1$ for every subset $S$. Under $f^{\rm lin}$ and $f^{\rm heat}$ the per-agent attributions are non-negative and $R_G^S$ is a genuine fraction of the total; under the concentration indicators $f^{\rm var}$ and $f^{\rm gini}$ below-mean agents receive negative attribution (Appendix~\ref{app:f-derivations}), so $R_G^S$ is a \emph{net} (signed) share that can in principle leave $[0,1]$. All ``share'' and ``mass'' statements for those two indicators are to be read as net shares. The flip we report below is a structural change in the values of $R_{G_{\rm top}}^S$, $R_{G_{\rm mid}}^S$, $R_{G_{\rm tail}}^S$ as $S$ moves from a biased small panel to the full panel.

\textbf{Four analytic value functions.} We instantiate $f$ as four analytic macro indicators, each capturing a different facet of how the panel's three features aggregate into a scalar:
\begin{itemize}\setlength\itemsep{0pt}
\item $f^{\rm lin}(z_S) = \tfrac{1}{n}\sum_{i\in S}(a_i+b_i+c_i)$, the linear additive baseline (Definition~\ref{def:linear} satisfied);
\item $f^{\rm heat}(z_S) = \log(1 + m_a(S)\, m_b(S)\, m_c(S))$, multiplicative reach $\times$ activity $\times$ resonance with logarithmic saturation, in the spirit of BM25~\citep{robertson1994some};
\item $f^{\rm var}(z_S) = \tfrac{1}{n}\sum_{i\in S}(g_i - \bar{g}_S)^2$ with $\bar{g}_S = m_a(S)+m_b(S)+m_c(S)$, the population variance of the additive composite;
\item $f^{\rm gini}(z_S) = \tfrac{1}{2n^2}\sum_{i,j\in S} |g_i - g_j|$, the absolute Gini mean difference~\citep{gini1912variabilita}.
\end{itemize}
The Aumann--Shapley path-integral attribution $\phi_{i,t}$ used throughout this section is defined formally in Section~\ref{sec:theory} (Equation~\ref{eq:phi-it}); the per-$f$ derivations are collected in Appendix~\ref{app:f-derivations} (Equations~\ref{eq:phi-lin}--\ref{eq:phi-gini}). None of the four involves any learned parameter, so attribution at every scale is independent of training noise.

\textbf{Sampling protocols.} A subset $S$ of size $N$ is drawn either by a \emph{biased} protocol or by the \emph{random} protocol. The biased protocol fixes a visibility pool of the top $5\%$ of users by a composite visibility score (followers plus topic engagement; Appendix~\ref{app:strategies}) and draws $N$ users uniformly from this pool. The random protocol draws $N$ users uniformly from the full $|[N]|$. Both protocols collapse to the full panel at $N = |[N]|$. The biased protocol is a plausible stand-in for how small-$N$ LLM-powered MAS case studies often construct their panels; we discuss two further biased variants and a complete dose-response sweep in Appendix~\ref{app:strategies}. We draw $10$ independent subsamples per $(N, \text{strategy}, f, \text{topic})$ cell, except at the full panel (a unique subset, single run), and report means across runs.

\subsection{Small-scale attribution misallocates the macro indicator}
\label{sec:finding-flip}

\begin{figure}[!t]
\centering
\includegraphics[width=\linewidth]{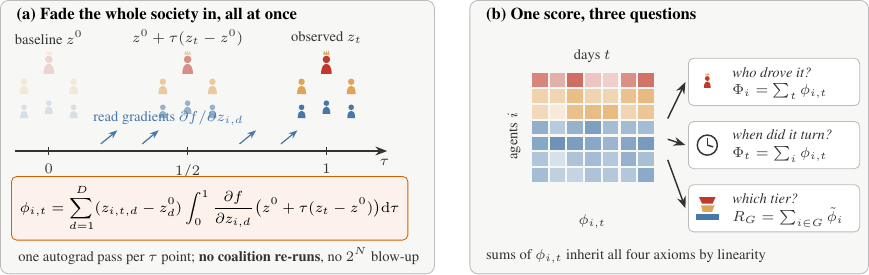}
\caption{\textbf{Aumann--Shapley path-integral attribution and its aggregation.} \textbf{(a)} Features fade in simultaneously from the baseline $z^0$ along the straight path $z^0 + \tau(z_t - z^0)$; agent $i$'s attribution $\phi_{i,t}$ integrates its own gradient along this path, with no coalition re-runs. \textbf{(b)} By linearity, the scores aggregate into the agent total $\Phi_i$ (\emph{who drove it}), the step total $\Phi_t$ (\emph{when did it turn}), and the tier share $R_G$.}
\label{fig:method}
\end{figure}

\begin{table}[t]
\caption{Top-tier relative attribution share $R_{G_{\rm top}}^S$ across five topics (columns) under three nonlinear value functions (row blocks) and four sampling protocols at $N = 10^2$. Each cell reports $R_{G_{\rm top}}^S$ in percent followed by $\Delta R_{G_{\rm top}} := R_{G_{\rm top}}^S - R_{G_{\rm top}}^{\rm full}$ in pp; the FULL row gives $R_{G_{\rm top}}^{\rm full}$. Small-panel rows are means over 10 seeds. The full three-tier breakdown and the linear baseline $f^{\rm lin}$ are in Table~\ref{tab:flip-app-cross-f}, and seed standard deviations are summarized in Appendix~\ref{app:topics}.}
\label{tab:flip-main}
\centering
\footnotesize
\setlength{\tabcolsep}{4pt}
\begin{tabular}{lccccc}
\toprule
 & Mythos & Trump-Tariffs & The Masters & Earth Day & WrestleMania \\
 & (tech) & (politics) & (sports) & (society) & (entertainment) \\
\midrule
\multicolumn{6}{l}{\emph{$f^{\rm heat}$ \quad multiplicative-saturating heat}} \\
$\rm BIAS_v$        & $11.3\,(+7.6)$  & $12.1\,(+5.6)$  & $10.8\,(+6.4)$  & $14.7\,(+9.3)$  & $7.8\,(+6.0)$  \\
$\rm BIAS_{xf}$     & $5.8\,(+2.1)$   & $10.5\,(+3.9)$  & $6.1\,(+1.8)$   & $9.3\,(+3.9)$   & $2.2\,(+0.4)$  \\
$\rm BIAS_{top}$    & $5.7\,(+2.1)$   & $9.9\,(+3.3)$   & $7.4\,(+3.1)$   & $7.9\,(+2.5)$   & $2.2\,(+0.4)$  \\
$\rm RAND$          & $0.8\,(-2.9)$   & $2.4\,(-4.2)$   & $7.4\,(+3.0)$   & $0.8\,(-4.6)$   & $0.8\,(-1.0)$  \\
FULL                & $3.7$           & $6.5$           & $4.4$           & $5.4$           & $1.8$          \\
\midrule
\multicolumn{6}{l}{\emph{$f^{\rm var}$ \quad population variance}} \\
$\rm BIAS_v$        & $78.3\,(+68.6)$ & $70.7\,(+61.0)$ & $79.3\,(+69.7)$ & $74.9\,(+65.2)$ & $79.5\,(+69.9)$ \\
$\rm BIAS_{xf}$     & $25.2\,(+15.6)$ & $45.7\,(+36.0)$ & $28.1\,(+18.5)$ & $32.4\,(+22.7)$ & $10.1\,(+0.5)$  \\
$\rm BIAS_{top}$    & $23.3\,(+13.7)$ & $40.8\,(+31.0)$ & $34.4\,(+24.8)$ & $28.5\,(+18.9)$ & $9.8\,(+0.2)$   \\
$\rm RAND$          & $7.3\,(-2.3)$   & $7.4\,(-2.4)$   & $7.6\,(-2.0)$   & $7.3\,(-2.3)$   & $7.3\,(-2.2)$   \\
FULL                & $9.7$           & $9.8$           & $9.7$           & $9.7$           & $9.6$           \\
\midrule
\multicolumn{6}{l}{\emph{$f^{\rm gini}$ \quad absolute Gini mean difference}} \\
$\rm BIAS_v$        & $63.5\,(+57.8)$ & $60.8\,(+55.1)$ & $62.4\,(+56.7)$ & $62.2\,(+56.4)$ & $65.4\,(+59.7)$ \\
$\rm BIAS_{xf}$     & $18.6\,(+12.8)$ & $36.6\,(+30.8)$ & $21.0\,(+15.2)$ & $26.3\,(+20.5)$ & $7.3\,(+1.6)$   \\
$\rm BIAS_{top}$    & $18.4\,(+12.7)$ & $33.3\,(+27.6)$ & $25.4\,(+19.7)$ & $21.5\,(+15.8)$ & $6.9\,(+1.2)$   \\
$\rm RAND$          & $4.5\,(-1.2)$   & $4.5\,(-1.3)$   & $4.5\,(-1.2)$   & $4.5\,(-1.2)$   & $4.5\,(-1.2)$   \\
FULL                & $5.7$           & $5.8$           & $5.7$           & $5.7$           & $5.7$           \\
\bottomrule
\end{tabular}
\end{table}

\textbf{Full scale: the long tail dominates.}
On Mythos at $N = 1.67 \times 10^6$ under $f^{\rm heat}$, $G_{\rm tail}$ carries $76.4\%$ of the within-panel attribution, $G_{\rm mid}$ $19.9\%$, and $G_{\rm top}$ only $3.7\%$ (Table~\ref{tab:flip-main}, FULL row). Per-agent attribution is largest inside $G_{\rm top}$, but the $90\%$ of agents in $G_{\rm tail}$ jointly carry the majority of the total: in Figure~\ref{fig:overview}, each tail bin is individually pale, but the tail bins together cover most of the colored area.

\textbf{Biased small panel: the upper tiers take over.}
At $N = 10^2$ under the biased protocol, the same attribution gives a structurally opposite answer (Table~\ref{tab:flip-main}, $\rm BIAS_v$ row): on Mythos, $G_{\rm top}$ triples to $11.3\%$, $G_{\rm mid}$ nearly doubles to $36.7\%$, and $G_{\rm tail}$ collapses to $52.0\%$. The combined upper-tier share rises from $24\%$ at full scale to $48\%$, a $+24$\,pp shift, with $G_{\rm mid}$ contributing more than $G_{\rm top}$. The random control at the same $N$ instead recovers the population fractions $(1\%, 9\%, 90\%)$ on average,\footnote{The isolated $\rm RAND$ value of $7.4\%$ on The Masters is an artifact of $f^{\rm heat}$'s saturating product at $N = 10^2$, where one high-follower agent can dominate $H = m_a m_b m_c$; all other topics and value functions leave $\rm RAND$ at or slightly below the full-panel value, so the bias-vs-random gap remains the structural signal.} confirming that the upper-tier \emph{inflation} is specific to the biased convenience sample. Note that recovering the population fraction is not the same as recovering the full-scale attribution: at $N = 10^2$ the random panel's $0.8\%$ sits below the full-scale $3.7\%$, so random small panels undershoot the top tier while biased panels overshoot it, and only scale itself closes the gap (Appendix~\ref{app:numbers}).

\textbf{The flip is universal across topics and $f$.}
The pattern of Table~\ref{tab:flip-main} repeats across all five topics and all four value functions: $\rm BIAS_v$ inflates $R_{G_{\rm top}}^S$ above the full-panel value in $20/20$ cells, the milder biased protocols sit between $\rm BIAS_v$ and $\rm RAND$, and the random control produces no inflation, staying at or below the full-panel value. The linear baseline $f^{\rm lin}$ shifts $R_{G_{\rm top}}^S$ by $\sim 20$\,pp; the concentration-flavored $f^{\rm var}$ and $f^{\rm gini}$ amplify the shift to $55$--$70$\,pp (Appendix~\ref{app:topics}). The flip is therefore not exclusive to nonlinear $f$: it is already $\approx 20$\,pp under linearity. What nonlinearity adds is a roughly $3\times$ larger magnitude and the loss of the post-hoc-rescaling escape (Theorem~\ref{thm:asb}, Table~\ref{tab:theorem-test}).

\subsection{Per-topic temporal dynamics emerge only at full scale}
\label{sec:finding-universal}

The flip also concerns when each tier contributes. Figure~\ref{fig:overview} shows that the per-day decomposition takes a topic-specific shape on the full panel: the peak day, the dominant tier on that day, and the overall shape differ from topic to topic in ways that match the underlying real-world events.

The sports topic peaks on April~13 UTC, as discussion of The Masters Sunday
final round (April~12, US~Eastern) propagates through the firehose, with
$R_{G_{\rm top}} = 4.4\%$ on that day, four times the $1\%$ population share.
The entertainment topic peaks on April~19 UTC, after WrestleMania night one
(April~18, US~Pacific), with $R_{G_{\rm tail}} = 82.1\%$, the highest tail share of
any topic-day combination, consistent with fan reactions and amateur reposts
driving event-aftermath engagement. The society topic peaks on April~20, two
days before Earth Day, with the largest top-tier share of any topic peak
($R_{G_{\rm top}} = 8.0\%$), consistent with media and advocacy accounts leading
the run-up coverage. The political and technology topics show flatter,
multi-modal trajectories without a single dominant peak (peak-day numerics:
Table~\ref{tab:peak-days}, Appendix~\ref{app:topic-dynamics}). Across all five
topics the biased small panel would predict the same shape, ``the upper tier
drove it, every day,'' regardless of which event was actually responsible.

Conclusions of the form ``agents of type $A$ drove emergence $E$ on day $t$'' drawn from practitioner-realistic small-$N$ panels can thus disagree structurally with the same conclusions at the phenomenon's realistic scale. Section~\ref{sec:theory} shows why the disagreement is structural, and gives the linear-time attribution that makes full-scale evaluation possible.

\begin{figure}[!t]
\centering
\includegraphics[width=\linewidth]{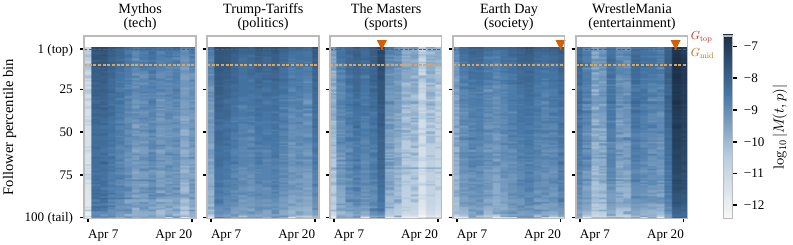}
\caption{\textbf{At full scale the long tail carries the attribution mass with topic-specific per-day structure.} Per-day attribution at full scale ($N = 1.67 \times 10^6$) on five Bluesky topics (April 7--20, 2026), $f = f^{\rm heat}$: rows are agents sorted by follower percentile (high at top, low at bottom), columns are days, and cell color is $\log_{10} |M(t,p)|$, the within-bin attribution mass (non-negative under $f^{\rm heat}$). Dashed lines: $G_{\rm top}$ ($1\%$) and $G_{\rm mid}$ ($10\%$) boundaries; triangles: event-peak days (Table~\ref{tab:peak-days}); bin-level construction of $M(t,p)$: Appendix~\ref{app:supercube}.}
\label{fig:overview}
\end{figure}

\section{Method: Attributing Emergence in Million-Agent Systems}
\label{sec:theory}
\label{sec:method}

Section~\ref{sec:finding} reports a qualitative flip in attribution as the sample size $N$ grows. A natural first reaction is that the flip is a measurement artifact: perhaps the two attributions differ by a single global scalar that one could divide out. This section shows that this hope is structurally limited: a global rescaling factor is guaranteed exactly when the macro indicator is linear over agents (Theorem~\ref{thm:asb}), and on our data reconciliation fails for every nonlinear indicator we measure (Table~\ref{tab:theorem-test}). We then derive linear-time Aumann--Shapley attributions for the four value functions of Section~\ref{sec:finding-setup}, removing the combinatorial cost of standard Shapley computation and letting the framework reach $N = 10^6$ in milliseconds.

\subsection{Aumann--Shapley path-integral attribution}
\label{sec:theory-setup}

Consider an LLM-powered MAS with $N$ agents and feature vectors $z_1, \dots, z_N \in \R^D$, where $z_i$ records what agent $i$ does and $D$ is the feature dimension. Let $[N] := \{1, \dots, N\}$, and let $f_n : \R^{n \times D} \to \R$ be a real-valued macro indicator evaluated on $n$ agents. The indicator may be evaluated on the full population $[N]$ or on a subset $S \subseteq [N]$ of size $n < N$, so we treat $f_n$ as a family $\{f_n\}_{n \geq 1}$ under two regularity assumptions.
\begin{assumption}[Twice differentiability]\label{as:f1}
Each $f_n$ is of class $C^2$ on $\R^{n \times D}$.
\end{assumption}
\begin{assumption}[Permutation invariance]\label{as:f2}
For every permutation $\pi$ of $[n]$, $f_n(z_{\pi(1)}, \dots, z_{\pi(n)}) = f_n(z_1, \dots, z_n)$.
\end{assumption}
Both assumptions hold for the macro indicators in standard use, including mean activity, variance, the Gini coefficient, cascade depth~\citep{goel2016structural}, and DeepSet value networks~\citep{zaheer2017deepsets, lee2019settransformer}, once non-smooth components such as $\max$ or indicator functions are replaced by their standard smooth surrogates.\footnote{$f^{\rm gini}$ is the one indicator we use in non-smooth form: $C^2$ except at ties, with a gradient that is rank-constant along the linear ray, so its Aumann--Shapley value is well-defined in closed form (Appendix~\ref{app:f-derivations}); we treat it as an empirical probe of the nonlinear regime rather than a literal instance of the $C^2$ hypothesis.}

To attribute the change in $f_n$ to individual agents we adopt the Aumann--Shapley value~\citep{aumann1974values} in its cost-sharing form~\citep{billera1982allocation}: each agent is a continuum of participation units faded in along the diagonal of the participation cube. It satisfies efficiency, symmetry, dummy, and linearity (Appendix~\ref{app:axioms}); on non-atomic games it is unique under the Aumann--Shapley axioms (dummy strengthened to positivity), and it coincides with discrete Shapley exactly when the participation game is multilinear~\citep{owen1972multilinear}, which holds for $f$ linear over agents but not in general. Computationally, the straight-path integral is Integrated Gradients~\citep{sundararajan2017axiomatic} with agents as features; the contribution of this section is the MAS instantiation, the closed forms that make it $O(N)$, and the cross-scale dichotomy below. The Aumann--Shapley value of agent $i$ is the integrated marginal effect of $f_n$ as agent $i$'s feature is faded in from a baseline $z^0 \in \R^D$ to its observed value $z_i$, along the straight path between them (Figure~\ref{fig:method}a).

\begin{definition}[Aumann--Shapley subset attribution]\label{def:as}
Fix $S \subseteq [N]$ with $|S| = n$. Let $z_S \in \R^{n \times D}$ collect the observed features on $S$ and let $z_S^0 \in \R^{n \times D}$ stack $n$ copies of the baseline $z^0$. The Aumann--Shapley attribution of agent $i \in S$ at scale $n$ is
\begin{equation}
\phi_i^S := \sum_{d=1}^{D} (z_{i,d} - z^0_d) \int_0^1 \frac{\partial f_n}{\partial z_{i,d}}\!\left(z_S^0 + \tau (z_S - z_S^0)\right) \mathrm{d}\tau \;\in\; \R,
\label{eq:as}
\end{equation}
where $z_{i,d}$ is the $d$-th coordinate of $z_i$ and $\tau \in [0,1]$ parameterizes the fade-in path. The normalized attribution is $\tilde{\phi}_i^S := \phi_i^S / \Delta v^S$ with $\Delta v^S := f_n(z_S) - f_n(z_S^0)$, the share of the macro change carried by agent $i$. We write $\tilde{\phi}_i := \tilde{\phi}_i^{[N]}$ for the share computed on the full population.
\end{definition}

When the data has a temporal dimension, we attribute per time step. Let $z_t := (z_{1,t}, \dots, z_{N,t}) \in \R^{N \times D}$ denote the agent state at time $t$, where $z_{i,t} \in \R^D$ is agent $i$'s feature vector at $t$. Setting $S = [N]$ in Definition~\ref{def:as} and indexing by $t$ gives the per-agent, per-step attribution
\begin{equation}
\phi_{i,t} := \sum_{d=1}^{D} (z_{i,t,d} - z^0_d) \int_0^1 \frac{\partial f}{\partial z_{i,d}}\!\left(z^0 + \tau (z_t - z^0)\right) \mathrm{d}\tau,
\label{eq:phi-it}
\end{equation}
where $f$ is the macro value function on $N$ agents, $z^0 \in \R^D$ is a baseline, and $T$ denotes the total number of time steps. By the Aumann--Shapley axioms (Appendix~\ref{app:axioms}), $\phi_{i,t}$ inherits efficiency, symmetry, dummy player, and linearity directly from Equation~\ref{eq:as}; the temporal index does not affect the guarantees. Aggregations follow by linearity (Figure~\ref{fig:method}b): the agent total $\Phi_i := \sum_{t=1}^T \phi_{i,t}$ answers \emph{who drove it}, the step total $\Phi_t := \sum_{i \in [N]} \phi_{i,t}$ answers \emph{when it turned}, and summing $\tilde\phi_i^S$ over $i \in S \cap G$ gives the tier share $R_G^S$ of Section~\ref{sec:finding}.

\subsection{Attribution Scaling Bias: a structural dichotomy}
\label{sec:theory-dichotomy}

The pair $(\tilde{\phi}_i^S, \tilde{\phi}_i)$ is what the post-processing hope concerns: a small-scale study reports the former, a million-agent study the latter, and the hope is a single scalar $c$, possibly depending on $S$ and on the configuration $z := (z_1, \dots, z_N) \in \R^{N \times D}$ but not on the agent index, with $\tilde{\phi}_i^S = c \cdot \tilde{\phi}_i$ for every $i \in S$. Whether such a $c$ exists depends on a single structural property of $f_n$.

\begin{definition}[Linear value function]\label{def:linear}
$f_n$ is \emph{linear} if there exists $\mu : \R^D \to \R$ of class $C^2$ such that
\[
f_n(z_1, \dots, z_n) = \frac{1}{n}\sum_{i=1}^n \mu(z_i).
\]
Linearity here is a property of how agents are aggregated, not of how individual features are processed: $\mu$ may itself be highly nonlinear.
\end{definition}

A linear indicator decomposes additively over agents, so when $f_n$ and $f_N$ share a generator $\mu$ the path integral in Equation~\ref{eq:as} collapses to $\mu(z_i) - \mu(z^0)$ and the two normalized attributions differ by an explicit agent-independent factor $c(S, z)$: the post-processing hope is realized exactly. Variance, the Gini coefficient, cascade depth, and DeepSet value networks all couple distinct agents and are nonlinear in the sense of Definition~\ref{def:linear}; mean activity, average sentiment, and aggregate engagement are linear.

\begin{theorem}[Attribution Scaling Bias]\label{thm:asb}
Let $\{f_n\}$ satisfy Assumptions~\ref{as:f1}--\ref{as:f2}. Fix $N > n \geq 2$.
\emph{(i) Linear case.} If $f_N$ and $f_n$ are both linear with a shared generator $\mu$ (Definition~\ref{def:linear}), then for every $S \subseteq [N]$ with $|S| = n$ and every $z \in \R^{N \times D}$ such that $\Delta v^S \neq 0$ and $\Delta v^{[N]} \neq 0$, there exists a scalar $c(S, z) \in \R$ depending on $S$ and $z$ but not on $i$ such that
\begin{equation}
\tilde{\phi}_i^S = c(S, z) \cdot \tilde{\phi}_i \quad \text{for every } i \in S.
\label{eq:reconcile}
\end{equation}
\emph{(ii) Nonlinear case.} The guarantee of part (i) does not extend beyond the linear class: there exist $C^2$, permutation-invariant nonlinear families $\{f_n\}$ (a minimal quadratic-interaction family, Equation~\ref{eq:counter-quad}, as well as the variance $f^{\rm var}$) together with features $z \in \R^{N \times D}$, a baseline $z^0 \in \R^D$, and a subset $S \subsetneq [N]$ with $|S| \geq 2$, satisfying $\Delta v^S \neq 0$ and $\Delta v^{[N]} \neq 0$, such that for every scalar $c \in \R$,
\begin{equation}
\max_{i \in S} \left| \tilde{\phi}_i^S - c \cdot \tilde{\phi}_i \right| > 0.
\label{eq:bias}
\end{equation}
\end{theorem}

The proof is in Appendix~\ref{app:proofs}. Part (i) says that under linearity a single rescaling factor always works; part (ii) says that this guarantee \emph{characterizes} the linear class, with explicit failing configurations already at $N = 3$. Part (ii) exhibits failing configurations for specific nonlinear families rather than asserting failure for every nonlinear family; for the three nonlinear indicators used in this paper, the residuals in Table~\ref{tab:theorem-test} certify that the failure is realized on the actual data.

\begin{table}[!t]
\caption{Residual $\varepsilon$ of the optimal agent-independent rescaling between $\tilde\phi^S$ at $N = 10^2$ visibility-biased and $\tilde\phi$ at the full panel, by topic and value function. Mean over ten subset seeds per cell.}
\label{tab:theorem-test}
\centering
\small
\setlength{\tabcolsep}{8pt}
\begin{tabular}{lcccc}
\toprule
Topic & $\varepsilon^{\rm lin}$ ($\times 10^{-8}$) & $\varepsilon^{\rm heat}$ & $\varepsilon^{\rm var}$ & $\varepsilon^{\rm gini}$ \\
\midrule
Mythos              & $3.4$ & $0.13$ & $0.89$ & $0.97$ \\
Trump-Tariffs       & $3.4$ & $0.13$ & $0.91$ & $0.97$ \\
The Masters         & $3.4$ & $0.30$ & $0.92$ & $0.98$ \\
Earth Day-Climate   & $3.5$ & $0.13$ & $0.89$ & $0.97$ \\
WrestleMania        & $3.3$ & $0.10$ & $0.91$ & $0.97$ \\
\bottomrule
\end{tabular}
\end{table}

\textbf{Scope.} Theorem~\ref{thm:asb} is a structural dichotomy: it does not bound the empirical mismatch on any particular dataset, nor claim the mismatch is monotone in nonlinearity. The empirical magnitude of the flip in Section~\ref{sec:finding} is dominated by the sampling structure of the small panel, not by the form of $f$: a flip of about $20$ percentage points already appears under the linear baseline $f^{\rm lin}$ (Appendix~\ref{app:topics}), where part (i) certifies that a single scalar would in principle reconcile the two scales. The role of the theorem is to remove the guarantee of such a generic repair once $f$ leaves the linear class, the regime occupied by the macro indicators researchers actually study.

\textbf{Empirical realization.} Fixing a subset $S$ and the per-agent shares restricted to it, the optimal scalar is $c^* := \arg\min_{c \in \R} \|\tilde\phi^S - c\, \tilde\phi\|_2^2$, with relative residual $\varepsilon := \|\tilde\phi^S - c^*\, \tilde\phi\|_2 / \|\tilde\phi^S\|_2$. Part (i) predicts $\varepsilon = 0$ for linear $f$; part (ii) removes the guarantee otherwise. On all five topics at $N = 10^2$ under the visibility-biased protocol (ten seeds each), $\varepsilon^{\rm lin}$ stays at floating-point precision while all three nonlinear $f$ give $\varepsilon$ uniformly positive (Table~\ref{tab:theorem-test}). The residual is primarily a property of $f$: $\varepsilon^{\rm var} \approx 0.89$--$0.92$, $\varepsilon^{\rm gini} \approx 0.97$--$0.98$ (cross-topic spread below $0.03$), and $\varepsilon^{\rm heat} \approx 0.10$--$0.13$ on four topics, with The Masters at $0.30$, the same event-concentrated cell behind the $\rm RAND$ anomaly of Table~\ref{tab:flip-main}. The two cases of the theorem appear in the same dataset under the same attribution method, separated only by the linearity of $f$ (test setup and a rank-correlation view: Appendix~\ref{app:theorem-test}). For $f^{\rm var}$ the residual is moreover exactly characterized: $\varepsilon = |m_S - m_N| \cdot \|P_v^\perp g_S\| / \|u\|$, the sine of the angle induced by the sampled panel's mean shift (Proposition~\ref{prop:residual}, Appendix~\ref{app:residual-law}). This law reproduces every measured $f^{\rm var}$ residual to $10^{-8}$ and explains the protocol ordering quantitatively: biased pools carry a fixed mean shift, so $\varepsilon \approx 0.9$ at every within-pool panel size, while random panels have shift $O_p(n^{-1/2})$ and their measured $\varepsilon$ decays from $0.046$ at $N = 10^2$ to $0.0016$ at $N = 10^5$.

\subsection{Aumann--Shapley attribution scales linearly to a million agents}
\label{sec:method-compute}
\label{sec:method-autograd}
\label{sec:method-baseline}

For each of the four value functions of Section~\ref{sec:finding-setup}, the integral in Definition~\ref{def:as} reduces to an explicit per-agent expression: $f^{\rm lin}$ has a constant integrand, $f^{\rm var}$ has an integrand linear in $\tau$, $f^{\rm gini}$ has an integrand that is rank-determined and so constant along the linear ray, and $f^{\rm heat}$ has a rational integrand whose $\tau$-integral is $\log(1+H)/(3H)$ with $H := m_a(S)\,m_b(S)\,m_c(S)$. The resulting expressions (Equations~\ref{eq:phi-lin}--\ref{eq:phi-gini}, Appendix~\ref{app:f-derivations}) are used throughout the main-text experiments. We verify them against $K = 300$-step numerical path integration with \texttt{torch.autograd}: per-agent MAE below $10^{-7}$, Spearman $\rho = 1.0$, efficiency violation below $10^{-12}$ across $N \in \{10, 10^2, \dots, 10^5\}$ (Table~\ref{tab:synthetic-accuracy}, Appendices~\ref{app:axioms} and~\ref{app:f-derivations}).

\textbf{General-$f$ fallback.} For value functions outside the four analytic cases above, we approximate Equation~\ref{eq:phi-it} by midpoint discretization:
\begin{equation}
\phi_{i,t} \approx \sum_{d=1}^{D} (z_{i,t,d} - z^0_d) \cdot \frac{1}{K} \sum_{k=1}^{K} \frac{\partial f}{\partial z_{i,d}}\!\left(z^0 + \frac{k - 1/2}{K}(z_t - z^0)\right).
\label{eq:phi-discrete}
\end{equation}
Each integration point costs one $\texttt{autograd}$ backward pass on $f$, which yields gradients for every agent and coordinate at the price of a single forward evaluation, so attribution across $T$ steps costs $O(T \cdot K \cdot C_f)$ and the combinatorial $2^N$ explosion of coalition-based methods is gone. The midpoint rule converges at the $O(1/K^2)$ quadrature rate; convergence numerics and the default $K = 30$ are in Appendix~\ref{app:integration}.

\begin{wrapfigure}{r}{0.46\linewidth}
\centering
\vspace{-1.2\baselineskip}
\includegraphics[width=\linewidth]{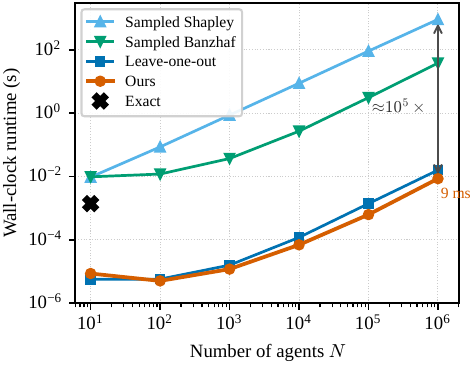}
\caption{\textbf{Attribution scales linearly in $N$, reaching $10^6$ agents in $9$\,ms.} Wall-clock (log--log) on Mythos, $f^{\rm heat}$. Sampled estimators use $m{=}1000$; exact methods are infeasible at $N \geq 10^2$.}
\label{fig:runtime}
\vspace{-1.0\baselineskip}
\end{wrapfigure}

\textbf{Empirical scaling.} On a single CPU node (Apple M4) our attribution at $N = 10^6$ on the Bluesky panel takes $\sim 9$\,ms for $f^{\rm heat}$ (other $f$ scale similarly; $f^{\rm gini}$ carries an extra $O(N \log N)$ sorting term), versus $\sim 9 \times 10^{2}$\,s for sampled Shapley with $m = 1000$ samples on the same hardware (measured directly at $N = 10^6$; cf.\ the $91$\,s at $N = 10^5$), a gap of three to five orders of magnitude across $N \in \{10, 10^2, 10^6\}$ (Figure~\ref{fig:runtime}); leave-one-out matches the scaling but keeps only three of the four axioms, and exact Shapley/Banzhaf are infeasible beyond $N \approx 100$ (full six-method wall-clock table: Table~\ref{tab:wallclock}, Appendix~\ref{app:synthetic}).

\textbf{Baseline choice.} We use the zero baseline $z^0 = \mathbf{0}$ throughout; the population-mean baseline $z^0 = \bar{z}$ gives an essentially identical agent ranking (top-$10$ Jaccard $1.00$, Kendall $\tau \approx 1.00$), and a baseline robustness sweep is in Appendix~\ref{app:baselines}, with path-choice and value-function-sensitivity sweeps consolidated in Appendix~\ref{app:robustness}.

\section{Conclusion}
\label{sec:conclusion}

Emergence attribution in LLM-powered MAS disagrees structurally between visibility-biased small panels and the full system, and only linear macro indicators guarantee a global-rescaling repair; the guarantee provably fails outside the linear class and empirically fails on our panel for all three nonlinear indicators (Theorem~\ref{thm:asb}, Table~\ref{tab:theorem-test}). The Aumann--Shapley framework behind the study keeps all four axioms, reaches a million agents in milliseconds, agrees with sampled Shapley where the latter is feasible (Appendix~\ref{app:llmmas}), and extends to a learned value network at the full panel (Appendix~\ref{app:deepset}).

\textbf{Limitations.} Theorem~\ref{thm:asb} is qualitative except under $f^{\rm var}$ (Proposition~\ref{prop:residual}), and the empirical magnitude is dominated by the panel's sampling structure rather than by $f$ (Section~\ref{sec:theory-dichotomy}). The result concerns biased convenience samples, not small-$N$ inference in general (more in Appendix~\ref{app:limitations}).

\bibliographystyle{plainnat}
\bibliography{references}

\newpage
\appendix

\section{Proof of Theorem~\ref{thm:asb} and Hessian characterization}
\label{app:proofs}

This appendix gives the full proof of Theorem~\ref{thm:asb}. Section~\ref{app:proofs-hessian} establishes a Hessian characterization of linearity that the proof relies on. Section~\ref{app:proofs-linear} proves part (i) of the theorem (linear reconciliation). Section~\ref{app:proofs-nonlinear} proves part (ii) by an explicit counterexample of minimal size $N = 3$.

Throughout, we work under Assumptions~\ref{as:f1} ($f_n$ is $C^2$) and~\ref{as:f2} (permutation invariance), and use the notation of Section~\ref{sec:theory-setup}: $[N] = \{1, \dots, N\}$, $S \subseteq [N]$ with $|S| = n$, $z_i \in \R^D$, $z^0 \in \R^D$, $\delta_i := z_i - z^0$, $z_S^0$ stacks $n$ copies of $z^0$, and $\Delta v^S := f_n(z_S) - f_n(z_S^0)$.

\subsection{Hessian characterization of linearity}
\label{app:proofs-hessian}

\begin{lemma}[Hessian characterization]\label{lem:hessian}
Fix $n \geq 2$ and let $f_n : \R^{n \times D} \to \R$ satisfy Assumptions~\ref{as:f1}--\ref{as:f2}. Then $f_n$ is linear in the sense of Definition~\ref{def:linear} if and only if
\begin{equation}
\frac{\partial^2 f_n}{\partial z_{i,d}\, \partial z_{j,d'}}(z) \equiv 0 \quad \text{for all } i \neq j \in [n],\; d, d' \in [D],\; z \in \R^{n \times D}.
\label{eq:hessian-block}
\end{equation}
\end{lemma}

\begin{proof}
$(\Rightarrow)$ If $f_n(z) = \tfrac{1}{n} \sum_{i=1}^n \mu(z_i)$ for some $\mu \in C^2(\R^D)$, then $\partial f_n / \partial z_{i,d}(z) = (1/n) (\partial_d \mu)(z_i)$ depends only on $z_i$. Differentiating with respect to $z_{j,d'}$ for $j \neq i$ gives zero, so~\eqref{eq:hessian-block} holds.

$(\Leftarrow)$ Assume~\eqref{eq:hessian-block}. By~\eqref{eq:hessian-block}, for each fixed $i$ the gradient $\nabla_{z_i} f_n(z)$ depends on $z$ only through the slot $z_i$; write $\nabla_{z_i} f_n(z) = \xi_i(z_i)$ with $\xi_i : \R^D \to \R^D$.

\emph{(a) Each $\xi_i$ is a gradient field.} Because $f_n \in C^2$, Schwarz's theorem yields $\partial_{d'} (\xi_i)_d = \partial_d (\xi_i)_{d'}$, i.e.\ $J_{\xi_i}$ is symmetric on $\R^D$. By the Poincar\'e lemma on the simply connected domain $\R^D$, there exists $\Psi_i \in C^2(\R^D)$ with $\xi_i = \nabla \Psi_i$.

\emph{(b) All $\xi_i$ coincide.} Permutation invariance (Assumption~\ref{as:f2}) gives $f_n(z) = f_n(z^{(i,j)})$, where $z^{(i,j)}$ is the configuration with positions $i$ and $j$ swapped. Differentiating both sides with respect to $z_i$ and using~\eqref{eq:hessian-block} (so that $\xi_i$ and $\xi_j$ each depend only on the value sitting in their respective slots) yields $\xi_i(y) = \xi_j(y)$ for all $y \in \R^D$. Write the common map as $\xi$ and select potentials $\Psi_i$ all equal to a single $\Psi \in C^2(\R^D)$ with $\Psi(z^\star) = 0$ for some fixed $z^\star \in \R^D$.

\emph{(c) Integrate slot by slot.} Starting from $(z^\star, \dots, z^\star)$, fade each coordinate $i$ in turn from $z^\star$ to $z_i$. By (a) and (b), each segment depends only on $z_i$ via $\xi$, so by the fundamental theorem of calculus
\[
f_n(z_1, \dots, z_i, z^\star, \dots, z^\star) - f_n(z_1, \dots, z_{i-1}, z^\star, \dots, z^\star) = \Psi(z_i) - \Psi(z^\star) = \Psi(z_i).
\]
Telescoping over $i = 1, \dots, n$ gives $f_n(z) - f_n(z^\star, \dots, z^\star) = \sum_{i=1}^n \Psi(z_i)$. Setting $\mu(y) := n \,\Psi(y) + f_n(z^\star, \dots, z^\star)$ rearranges this to $f_n(z) = (1/n) \sum_i \mu(y_i)$, recovering Definition~\ref{def:linear}. \qedhere
\end{proof}

\subsection{Proof of Theorem~\ref{thm:asb} (i): linear reconciliation}
\label{app:proofs-linear}

Assume $f_N$ and $f_n$ are both linear with shared generator $\mu$, so
\[
f_n(z_S) = \frac{1}{n} \sum_{j \in S} \mu(z_j), \qquad f_N(z) = \frac{1}{N} \sum_{j \in [N]} \mu(z_j).
\]
For $i \in S$, $\partial f_n / \partial z_{i,d}(z_S) = (1/n)(\partial_d \mu)(z_i)$, and this expression depends only on $z_i$. Substituting into Definition~\ref{def:as},
\[
\phi_i^S = \sum_d \delta_{i,d} \int_0^1 \frac{1}{n} (\partial_d \mu)(z^0 + \tau \delta_i) \,\mathrm{d}\tau = \frac{1}{n} \int_0^1 \frac{\mathrm{d}}{\mathrm{d}\tau} \mu(z^0 + \tau \delta_i) \,\mathrm{d}\tau = \frac{\mu(z_i) - \mu(z^0)}{n}.
\]
The same calculation at scale $N$ gives $\phi_i = (\mu(z_i) - \mu(z^0)) / N$ for $i \in [N]$. Summing the first identity over $i \in S$ and the second over $i \in [N]$:
\[
\Delta v^S = \frac{1}{n} \sum_{j \in S} [\mu(z_j) - \mu(z^0)], \qquad \Delta v^{[N]} = \frac{1}{N} \sum_{j \in [N]} [\mu(z_j) - \mu(z^0)].
\]
Both are non-zero by hypothesis, so the normalized attributions are well-defined:
\[
\tilde\phi_i^S = \frac{\mu(z_i) - \mu(z^0)}{\sum_{j \in S} [\mu(z_j) - \mu(z^0)]}, \qquad \tilde\phi_i = \frac{\mu(z_i) - \mu(z^0)}{\sum_{j \in [N]} [\mu(z_j) - \mu(z^0)]}.
\]
Taking the ratio,
\begin{equation}
\tilde\phi_i^S = c(S, z) \cdot \tilde\phi_i, \qquad c(S, z) = \frac{\sum_{j \in [N]} [\mu(z_j) - \mu(z^0)]}{\sum_{j \in S} [\mu(z_j) - \mu(z^0)]},
\label{eq:linear-c}
\end{equation}
which depends on $S$ and $z$ but not on $i \in S$. This is exactly~\eqref{eq:reconcile}. \qed

\subsection{Proof of Theorem~\ref{thm:asb} (ii): explicit counterexample}
\label{app:proofs-nonlinear}

Take $D = 1$ and $z^0 = 0$. Define
\begin{equation}
f_n(z_1, \dots, z_n) := \frac{1}{n^2} \sum_{1 \leq i < j \leq n} z_i z_j.
\label{eq:counter-quad}
\end{equation}
This $f_n$ is $C^2$ and permutation-invariant for every $n \geq 2$. For $i \neq j$,
\[
\frac{\partial^2 f_n}{\partial z_i\, \partial z_j} = \frac{1}{n^2} \neq 0,
\]
so by Lemma~\ref{lem:hessian} every $f_n$ in the family is nonlinear, satisfying the hypothesis of part (ii) for all relevant scales.

\textbf{Per-agent attribution.} Using $f_n(z) = (1/(2n^2))[(\sum_k z_k)^2 - \sum_k z_k^2]$,
\[
\frac{\partial f_n}{\partial z_i}(z) = \frac{1}{n^2} \sum_{k \neq i} z_k.
\]
Along the path $\tau z$ (since $z^0 = 0$),
\begin{equation}
\phi_i^S = z_i \int_0^1 \frac{\tau}{n^2} \sum_{\substack{k \in S \\ k \neq i}} z_k \,\mathrm{d}\tau = \frac{z_i}{2 n^2} \sum_{\substack{k \in S \\ k \neq i}} z_k, \quad n = |S|.
\label{eq:counter-phi}
\end{equation}
Note that $\phi_i^S$ depends on $S$ through the sum $\sum_{k \in S, k \neq i} z_k$, not only through $|S|$. This is the direct consequence of the non-vanishing off-diagonal Hessian.

\textbf{Numerical instantiation at $N = 3$.} Take $(z_1, z_2, z_3) = (1, 1, 2)$ and $S = \{1, 3\}$.

\emph{Full population.} Using~\eqref{eq:counter-phi} with $n = 3$:
\[
\phi_1 = \tfrac{1}{18}(1 + 2) = \tfrac{1}{6}, \quad \phi_2 = \tfrac{1}{18}(1 + 2) = \tfrac{1}{6}, \quad \phi_3 = \tfrac{2}{18}(1 + 1) = \tfrac{2}{9}.
\]
$\Delta v^{[3]} = (1/9)(1 \cdot 1 + 1 \cdot 2 + 1 \cdot 2) = 5/9$. Efficiency check: $1/6 + 1/6 + 2/9 = (3 + 3 + 4)/18 = 5/9$. Hence
\[
\tilde\phi_1 = \tilde\phi_2 = \frac{1/6}{5/9} = \frac{3}{10}, \qquad \tilde\phi_3 = \frac{2/9}{5/9} = \frac{4}{10}.
\]

\emph{Subset $S = \{1, 3\}$.} Using~\eqref{eq:counter-phi} with $n = 2$:
\[
\phi_1^S = \tfrac{1}{8} \cdot 2 = \tfrac{1}{4}, \quad \phi_3^S = \tfrac{2}{8} \cdot 1 = \tfrac{1}{4}.
\]
$\Delta v^S = (1/4)(1 \cdot 2) = 1/2$. Efficiency check: $1/4 + 1/4 = 1/2$. Hence
$\tilde\phi_1^S = \tilde\phi_3^S = 1/2$.

\emph{Failure of any agent-independent rescaling.} Suppose $c \in \R$ satisfies $\tilde\phi_i^S = c \cdot \tilde\phi_i$ for $i \in S$. From $i = 1$: $c = (1/2)/(3/10) = 5/3$. From $i = 3$: $c = (1/2)/(4/10) = 5/4$. These are unequal, so no such $c$ exists. Both $\Delta v^S = 1/2 \neq 0$ and $\Delta v^{[3]} = 5/9 \neq 0$, so the normalized attributions are well-defined. Hence \eqref{eq:bias} holds. \qed

\textbf{Remark on minimality.} The counterexample uses $N = 3$, which is minimal: at $N = 2$, the only proper non-empty subset $S \subsetneq [2]$ has size $1$, forcing $\tilde\phi_i^S = 1$ trivially for the unique $i \in S$, so the conclusion of (ii) is vacuous. Three agents is the smallest $N$ for which a nonlinear off-diagonal Hessian can manifest as a multi-agent reconciliation failure.

\textbf{Remark ($f^{\rm var}$ admits the same construction).} The variance indicator fails identically on the same configuration. Take $D = 1$, $z^0 = 0$, $(z_1, z_2, z_3) = (1, 1, 2)$, $S = \{1, 3\}$, and $f^{\rm var}_n(z) = \tfrac{1}{n}\sum_i (z_i - \bar z)^2$ with closed-form attribution $\phi_i = z_i (z_i - \bar z)/n$ (Equation~\ref{eq:phi-var}). At full scale, $\bar z = 4/3$: $\phi_1 = \phi_2 = -1/9$, $\phi_3 = 4/9$, $\Delta v^{[3]} = 2/9$, so $\tilde\phi_1 = -1/2$, $\tilde\phi_3 = 2$. On $S$, $\bar z_S = 3/2$: $\phi_1^S = -1/4$, $\phi_3^S = 1/2$, $\Delta v^S = 1/4$, so $\tilde\phi_1^S = -1$, $\tilde\phi_3^S = 2$. Agent $1$ demands $c = 2$; agent $3$ demands $c = 1$. No agent-independent scalar exists. (This example also illustrates the signed shares discussed in Section~\ref{sec:finding-setup}: normalized attributions under $f^{\rm var}$ can be negative or exceed $1$.)

\subsection{An exact residual law for the variance indicator}
\label{app:residual-law}

For $f^{\rm var}$ the dichotomy of Theorem~\ref{thm:asb} sharpens into a closed-form expression for the residual itself, which identifies the sampled panel's mean shift as the exact source of irreparability.

\begin{proposition}[Exact residual law for $f^{\rm var}$]\label{prop:residual}
Fix $S \subseteq [N]$ with $|S| = n \geq 2$, let $g_i$ be the additive composite of Section~\ref{sec:finding-setup}, and write $m_S := \tfrac{1}{n}\sum_{i \in S} g_i$, $m_N := \tfrac{1}{N}\sum_{i \in [N]} g_i$, and, as vectors in $\R^S$, $u_i := g_i (g_i - m_S)$ and $v_i := g_i (g_i - m_N)$. Assume $\Delta v^S \neq 0$ and $\Delta v^{[N]} \neq 0$. Then the best-scalar residual of Section~\ref{sec:theory-dichotomy} under $f^{\rm var}$ equals
\begin{equation}
\varepsilon \;=\; \sin \angle(u, v) \;=\; \frac{|m_S - m_N| \cdot \| P_{v}^{\perp} g_S \|_2}{\| u \|_2},
\label{eq:residual-law}
\end{equation}
where $g_S := (g_i)_{i \in S}$ and $P_{v}^{\perp}$ projects onto the orthogonal complement of $\mathrm{span}\{v\}$ in $\R^S$. In particular, $\varepsilon = 0$ if and only if $m_S = m_N$ or $g_S \in \mathrm{span}\{v\}$.
\end{proposition}

\begin{proof}
By Equation~\ref{eq:phi-var}, $\phi_i^S = g_i (g_i - m_S)/n$ and $\Delta v^S = \mathrm{Var}_S(g)$, so $\tilde\phi^S \propto u$; restricting the full-panel attribution to $S$ gives $\tilde\phi|_S \propto v$ by the same formula at scale $N$. The scalings are absorbed into $c$, so $\varepsilon = \min_c \|u - c v\|_2 / \|u\|_2 = \sin\angle(u, v)$. The identity $v_i = u_i + (m_S - m_N)\, g_i$ holds term by term, so the component of $u$ orthogonal to $v$ is $-(m_S - m_N) P_{v}^{\perp} g_S$, which yields Equation~\ref{eq:residual-law} and the zero characterization.
\end{proof}

Under uniform random sampling $m_S - m_N = O_p(n^{-1/2})$, so $\varepsilon \to 0$ at the $n^{-1/2}$ rate and random panels are asymptotically reparable; under any sampling scheme whose pool carries a fixed mean shift $\Delta \neq 0$, $\varepsilon$ stays bounded away from zero at every panel size drawn from inside the pool. The scaling bias under $f^{\rm var}$ is thus exactly a mean-shift phenomenon.

\textbf{Empirical verification.} We recompute Equation~\ref{eq:residual-law} from the raw features and the stored subset indices of every $f^{\rm var}$ run at $N = 10^2$ (five topics, four sampling protocols, ten seeds) and compare against the independently measured residuals: the maximum absolute deviation over all $200$ runs is $9.3 \times 10^{-9}$, i.e.\ the law holds at float32 storage precision. On Mythos the dose-response it predicts is: $\rm BIAS_v$, mean shift $|\delta| \approx 4.0$, $\varepsilon \approx 0.89$; $\rm BIAS_{xf}$ and $\rm BIAS_{top}$, $|\delta| \approx 2.7$--$2.8$, $\varepsilon \approx 0.51$; RAND, $|\delta| \approx 0.13$, $\varepsilon \approx 0.046$. Across panel sizes, the measured RAND residual falls as $0.046 \to 0.012 \to 0.0053 \to 0.0016$ from $N = 10^2$ to $10^5$, matching the $n^{-1/2}$ rate, while $\rm BIAS_v$ stays at $0.89$--$0.90$ through $N = 10^4$ and drops only when the pool is exhausted.

\textbf{Remark on generality.} Part (ii) as stated exhibits failing configurations for specific nonlinear families; it does not assert that \emph{every} nonlinear family fails on some admissible configuration. We conjecture the stronger statement holds, since the non-vanishing off-diagonal Hessian granted by Lemma~\ref{lem:hessian} generically propagates into the path integrals of Definition~\ref{def:as}; a fully general proof must rule out cancellation inside the integrals and is left open. For the indicators studied in this paper, the measured residuals of Table~\ref{tab:theorem-test} certify failure directly on the data.

\section{Aumann--Shapley axioms for the path-integral attribution}
\label{app:axioms}

We verify that the per-agent, per-step attribution
\[
\phi_{i,t} = \sum_d (z_{i,t,d} - z^0_d) \int_0^1 \frac{\partial f}{\partial z_{i,d}}\!\left(z^0 + \tau (z_t - z^0)\right) \mathrm{d}\tau
\]
of Equation~\ref{eq:phi-it} satisfies all four axioms identified by~\citet{aumann1974values} for non-atomic games and by~\citet{shapley1953value} for finite cooperative games. Throughout, we fix $t$ and write $\phi_i := \phi_{i,t}$, $z := z_t$.

\textbf{Efficiency.} Apply the chain rule to $g(\tau) := f(z^0 + \tau (z - z^0))$. Then $g'(\tau) = \sum_{i,d} (z_{i,d} - z^0_d) (\partial f / \partial z_{i,d})(z^0 + \tau (z - z^0))$. Integrating over $[0,1]$ and using the fundamental theorem of calculus,
\[
\sum_{i \in [N]} \phi_i = \int_0^1 g'(\tau) \,\mathrm{d}\tau = g(1) - g(0) = f(z) - f(z^0) = \Delta v.
\]

\textbf{Symmetry.} If two agents $i, j$ have identical features $z_i = z_j$ and $f$ is permutation-invariant, then $(\partial f / \partial z_{i,d})(z) = (\partial f / \partial z_{j,d})(z)$ at every point on the path, and the prefactor $z_{i,d} - z^0_d = z_{j,d} - z^0_d$ matches as well. Hence $\phi_i = \phi_j$.

\textbf{Dummy player.} If $f$ does not depend on $z_i$, then $(\partial f / \partial z_{i,d}) \equiv 0$, so the integrand vanishes identically and $\phi_i = 0$.

\textbf{Linearity.} If $f = \alpha f^{(1)} + \beta f^{(2)}$ for $\alpha, \beta \in \R$, the gradient is linear in $f$, so $\phi_i[f] = \alpha \,\phi_i[f^{(1)}] + \beta \,\phi_i[f^{(2)}]$. Hence the path-integral attribution is linear in the macro value function.

The temporal index $t$ commutes with all four axioms: each axiom holds at every fixed $t$, and grouping by identity ($\Phi_i = \sum_t \phi_{i,t}$) or by time ($\Phi_t = \sum_i \phi_{i,t}$) preserves the axioms by linearity.

\section{Integration error analysis}
\label{app:integration}

The midpoint discretization in Equation~\ref{eq:phi-discrete} approximates the path integral with $K$ uniformly spaced sample points. Standard quadrature theory gives, for a $C^2$ integrand $h(\tau)$,
\[
\left| \int_0^1 h(\tau) \,\mathrm{d}\tau - \frac{1}{K} \sum_{k=1}^K h\!\left(\frac{k - 1/2}{K}\right) \right| \leq \frac{\|h''\|_\infty}{24\, K^2},
\]
so the per-agent attribution error is $O(1/K^2)$ when $\partial^2 f / \partial z_{i,d}^2$ is bounded on the path. Summed over $D$ coordinates and weighted by $|z_{i,d} - z^0_d|$, the upper bound on $|\hat\phi_i - \phi_i|$ remains $O(1/K^2)$ uniformly in $i$.

Empirically, on $f^{\rm heat}$ at Mythos with $N = 10^4$, the relative $L^1$ error of $\hat\phi_K$ against the analytic reference (Appendix~\ref{app:f-derivations}) decays as $1/K^2$ across $K \in \{5, 10, 20, 30, 50, 100, 300\}$, falling from $4.0 \times 10^{-3}$ at $K = 5$ to $1.1 \times 10^{-6}$ at $K = 300$ (Table~\ref{tab:K-convergence}). The wall-clock cost grows linearly in $K$ (after the first-call warm-up visible in the $K = 5$ row), so where no analytic expression is available we use $K = 30$ as a default: its integration error ($\approx 1.1 \times 10^{-4}$) sits two orders of magnitude below the seed-to-seed sampling variability of the small-panel $R_{G_{\rm top}}^S$ ($\approx 3 \times 10^{-2}$), and pushing to $K = 50$ or $100$ tightens the error to $4.0 \times 10^{-5}$ and $1.0 \times 10^{-5}$ at $1.5\times$ and $2.9\times$ wall-clock cost. For the four analytic value functions of Section~\ref{sec:finding-setup}, all main-text experiments use the per-$f$ expressions of Appendix~\ref{app:f-derivations} directly.

\begin{table}[!tbp]
\caption{Midpoint-rule convergence of the path-integral estimator against the analytic Aumann--Shapley value of $f^{\rm heat}$, evaluated on Mythos at $N = 10^4$, as a function of the number of integration points $K$. Wall-clock is single-CPU.}
\label{tab:K-convergence}
\centering
\small
\begin{tabular}{rcc}
\toprule
$K$ & Rel.\ $L^1$ error vs.\ analytic & Wall-clock (s) \\
\midrule
$5$   & $4.0 \times 10^{-3}$ & $5.7 \times 10^{-4}$ \\
$10$  & $1.0 \times 10^{-3}$ & $8.9 \times 10^{-5}$ \\
$20$  & $2.5 \times 10^{-4}$ & $1.3 \times 10^{-4}$ \\
$30$  & $1.1 \times 10^{-4}$ & $1.8 \times 10^{-4}$ \\
$50$  & $4.0 \times 10^{-5}$ & $2.7 \times 10^{-4}$ \\
$100$ & $1.0 \times 10^{-5}$ & $5.2 \times 10^{-4}$ \\
$300$ & $1.1 \times 10^{-6}$ & $1.4 \times 10^{-3}$ \\
\bottomrule
\end{tabular}
\end{table}

\section{Synthetic benchmark details}
\label{app:synthetic}

This appendix details the synthetic benchmark used to certify that the path-integral implementation reproduces the analytic Aumann--Shapley value, free of any stochastic estimator (complementing the four-function verification in Appendix~\ref{app:f-derivations}).

\textbf{Generation.} Each instance is parameterized by $(N, T, D, \text{seed})$. We use $N \in \{10, 50, 100, 500, 1000, 5000, 10000\}$ for the scaling sweep, fix $T = 100$ and $D = 5$, and average over five seeds. Features $z_{i,t,d}$ are drawn i.i.d.\ from $\mathrm{Uniform}(-1, 1)$. The baseline is $z^0 = \mathbf{0}$.

\textbf{Three value functions with analytic ground truth.}
\emph{Additive:} $f^{\mathrm{add}}_n(z) = \sum_{i,d} W_{i,d} z_{i,d}$, with weights $W_{i,d} \sim \mathcal{N}(0, 1)$. The Aumann--Shapley value is $\phi_i = \sum_d W_{i,d} z_{i,d}$. Note that $f^{\mathrm{add}}$ is not linear in the sense of Definition~\ref{def:linear} (and not permutation-invariant, so Assumption~\ref{as:f2} fails) because the weights depend on the agent index $i$. The path-integral attribution of Equation~\ref{eq:as} does not itself require permutation invariance to be defined; what $f^{\rm add}$ loses is only the symmetry axiom. It remains a useful sanity check for the path-integral implementation.
\emph{Quadratic with cross terms:} $f^{\mathrm{quad}}_n(z) = \sum_{i,d} Q_{i,d} z_{i,d}^2 + \tfrac{1}{2} s^\top C s$, where $s_i = \sum_d z_{i,d}$ and $C \in \R^{N \times N}$ has i.i.d.\ standard-Gaussian off-diagonal entries symmetrized as $C \gets (C + C^\top)/2$. The per-agent attribution follows analytically by direct integration along the linear path.
\emph{Softplus aggregator:} $f^{\mathrm{nl}}_n(z) = \mathrm{softplus}(a \cdot s)/a$ with $a = 0.35$ and $s = \sum_{i,d} W_{i,d} z_{i,d}$, giving a smooth, globally non-additive aggregator. Ground truth is computed as a high-resolution path integral with $10000$ midpoints, accurate to floating-point precision.

\textbf{Methods compared.} Ours (path integral, $K = 30$); leave-one-out (LOO); sampled Shapley with $200$ permutation samples, following the Monte Carlo estimator used by~\citet{tang2026interpreting} for the same MAS attribution problem; sampled Banzhaf with $200$ coalition samples~\citep{wang2023banzhaf}; and two LLM-as-Judge variants, MAST~\citep{cemri2025mast} and Who\&When~\citep{zhang2025whoandwhen}, both prompt-based scorers operating on agent execution traces.

\textbf{Metrics.} \emph{MAE}: mean absolute error against ground truth. \emph{Cosine}: cosine similarity of attribution vectors. \emph{Efficiency error}: $|\sum_i \hat\phi_i - \Delta v|$. \emph{Top-$k$ Kendall $\tau$}: rank correlation on the top-$k$ agents (default $k = 5$).

\textbf{Results.} Tables~\ref{tab:synth-additive}--\ref{tab:synth-nonlinear} report mean values over five seeds. Path-integral attribution matches the analytic ground truth to numerical precision on the additive and quadratic functions ($\mathrm{MAE} \leq 10^{-13}$ and $\leq 10^{-15}$ respectively) and to discretization error on the softplus function ($\mathrm{MAE} \approx 4.1 \times 10^{-6}$). Sampled Shapley and Banzhaf retain low cosine error but exhibit visible MAE on quadratic and softplus inputs at $200$ samples, and LOO has much larger errors on the same. MAST and Who\&When, both prompt-based, do not reach the precision of structural attribution methods on this benchmark.

\begin{table}[!tbp]
\caption{Synthetic benchmark, additive $f$ (analytic per-agent ground truth). Mean over five seeds.}
\label{tab:synth-additive}
\centering
\small
\begin{tabular}{lcccc}
\toprule
Method & MAE & Cosine & Efficiency error & Top-$k$ $\tau$ \\
\midrule
Ours              & $1.15\!\times\!10^{-13}$ & $1.0000$  & $1.45\!\times\!10^{-15}$ & $1.00$ \\
LOO~\citep{luagent}             & $1.15\!\times\!10^{-13}$ & $1.0000$  & $2.33\!\times\!10^{-15}$ & $1.00$ \\
Sampled Shapley~\citep{tang2026interpreting}   & $1.15\!\times\!10^{-13}$ & $1.0000$  & $6.79\!\times\!10^{-15}$ & $1.00$ \\
Banzhaf~\citep{wang2023banzhaf}           & $1.15\!\times\!10^{-13}$ & $1.0000$  & $6.89\!\times\!10^{-15}$ & $1.00$ \\
MAST~\citep{cemri2025mast}              & $8.82\!\times\!10^{-1}$ & $0.0800$ & $3.81\!\times\!10^{0}$                   & $0.08$ \\
Who\&When~\citep{zhang2025whoandwhen}         & $8.55\!\times\!10^{-1}$ & $0.1770$ & $3.66\!\times\!10^{0}$                   & $0.04$ \\
\bottomrule
\end{tabular}
\end{table}

\begin{table}[!tbp]
\caption{Synthetic benchmark, quadratic $f$ with cross terms. Mean over five seeds.}
\label{tab:synth-quadratic}
\centering
\small
\begin{tabular}{lcccc}
\toprule
Method & MAE & Cosine & Efficiency error & Top-$k$ $\tau$ \\
\midrule
Ours              & $3.40\!\times\!10^{-16}$ & $1.0000$   & $4.12\!\times\!10^{-16}$  & $1.00$ \\
LOO~\citep{luagent}             & $1.70\!\times\!10^{-1}$  & $0.9560$   & $8.51\!\times\!10^{-1}$   & $0.96$ \\
Sampled Shapley~\citep{tang2026interpreting}   & $1.17\!\times\!10^{-2}$  & $0.9996$  & $9.76\!\times\!10^{-16}$  & $1.00$ \\
Banzhaf~\citep{wang2023banzhaf}           & $1.19\!\times\!10^{-2}$  & $0.9995$  & $4.41\!\times\!10^{-2}$   & $1.00$ \\
MAST~\citep{cemri2025mast}              & $5.71\!\times\!10^{-1}$ & $-0.1960$  & $3.41\!\times\!10^{0}$                    & $0.04$ \\
Who\&When~\citep{zhang2025whoandwhen}         & $5.75\!\times\!10^{-1}$ & $-0.1260$  & $3.91\!\times\!10^{0}$                    & $0.16$ \\
\bottomrule
\end{tabular}
\end{table}

\begin{table}[!tbp]
\caption{Synthetic benchmark, softplus aggregator (nonlinear $f$). Mean over five seeds.}
\label{tab:synth-nonlinear}
\centering
\small
\begin{tabular}{lcccc}
\toprule
Method & MAE & Cosine & Efficiency error & Top-$k$ $\tau$ \\
\midrule
Ours              & $4.13\!\times\!10^{-6}$  & $1.0000$ & $2.39\!\times\!10^{-5}$ & $1.00$ \\
LOO~\citep{luagent}             & $7.94\!\times\!10^{-2}$  & $0.9810$       & $4.02\!\times\!10^{-1}$ & $0.88$ \\
Sampled Shapley~\citep{tang2026interpreting}   & $6.84\!\times\!10^{-3}$  & $0.9998$      & $7.05\!\times\!10^{-16}$ & $0.96$ \\
Banzhaf~\citep{wang2023banzhaf}           & $7.78\!\times\!10^{-3}$  & $0.9998$      & $2.60\!\times\!10^{-2}$ & $0.96$ \\
MAST~\citep{cemri2025mast}              & $4.35\!\times\!10^{-1}$ & $0.2570$       & $2.67\!\times\!10^{0}$                  & $0.48$ \\
Who\&When~\citep{zhang2025whoandwhen}         & $3.65\!\times\!10^{-1}$ & $0.5010$       & $2.02\!\times\!10^{0}$                  & $0.68$ \\
\bottomrule
\end{tabular}
\end{table}

\textbf{Scaling experiment.} On Mythos with $f^{\rm heat}$, we report wall-clock runtime as $N$ varies, computed exactly when feasible and otherwise extrapolated linearly from the largest exact calibration point. The hardware is an Apple M4 CPU node with $10$ cores and $16$~GB of memory. Sampled Shapley and Banzhaf each use $1{,}000$ samples. Numbers in Table~\ref{tab:scalability} confirm the linear-in-$N$ growth of our method and the impossibility of exact coalition-based estimators beyond $N \approx 100$.

\begin{table}[!tbp]
\caption{Full per-$N$ wall-clock runtime in seconds on Mythos with $f^{\rm heat}$, $T = 14$, $D = 3$, $3$ repeats, single CPU node. ``Ours'' is the analytic Aumann--Shapley attribution; sampled Shapley and sampled Banzhaf use $m = 1000$ samples. Exact Shapley and Banzhaf evaluate $f$ on every coalition of $[N]$ and are infeasible at $N \geq 10^2$ ($\Omega(2^N)$ work). }
\label{tab:scalability}
\label{tab:wallclock}
\centering
\small
\setlength{\tabcolsep}{4pt}
\begin{tabular}{rcccccc}
\toprule
$N$    & Ours & LOO & Sampled Shapley & Sampled Banzhaf & Exact Shapley & Exact Banzhaf \\
\midrule
$10$   & $8.7\!\times\!10^{-6}$ & $5.7\!\times\!10^{-6}$ & $9.6\!\times\!10^{-3}$ & $9.8\!\times\!10^{-3}$ & $1.4\!\times\!10^{-3}$ & $1.4\!\times\!10^{-3}$ \\
$10^2$ & $5.1\!\times\!10^{-6}$ & $5.8\!\times\!10^{-6}$ & $8.6\!\times\!10^{-2}$ & $1.2\!\times\!10^{-2}$ & infeasible & infeasible \\
$10^3$ & $1.2\!\times\!10^{-5}$ & $1.6\!\times\!10^{-5}$ & $8.7\!\times\!10^{-1}$ & $3.7\!\times\!10^{-2}$ & infeasible & infeasible \\
$10^4$ & $7.0\!\times\!10^{-5}$ & $1.2\!\times\!10^{-4}$ & $8.8\!\times\!10^{0}$  & $2.7\!\times\!10^{-1}$ & infeasible & infeasible \\
$10^5$ & $6.3\!\times\!10^{-4}$ & $1.4\!\times\!10^{-3}$ & $9.1\!\times\!10^{1}$  & $3.1\!\times\!10^{0}$  & infeasible & infeasible \\
$10^6$ & $8.6\!\times\!10^{-3}$ & $1.6\!\times\!10^{-2}$ & $9.3\!\times\!10^{2}$  & $3.8\!\times\!10^{1}$ & infeasible & infeasible \\
\bottomrule
\end{tabular}
\end{table}

\section{Baseline robustness}
\label{app:baselines}

This appendix details the baseline ablation referenced in Section~\ref{sec:method-baseline}. Three baselines are compared on Bluesky topic Mythos at $N = 10^4$ with $f^{\rm heat}$ across five seeds: \emph{Zero} ($z^0 = \mathbf{0}$, the default), \emph{Population-mean} ($z^0 = \bar z$, the panel mean), and \emph{Pre-event} ($z^0 = z_{t=1}$, the features observed on the first day of the window).

\textbf{Setup.} For each baseline we compute attribution and rank agents by $|\Phi_i^{S}|$. We then measure: (i) top-$10$ Jaccard $J_{10}$ and top-$100$ Jaccard $J_{100}$ between full attribution rankings, (ii) Kendall $\tau$, and (iii) Spearman $\rho$, all on the agents in $S$.

\textbf{Results.} The Zero and Population-mean baselines agree almost exactly across all five seeds: $J_{10} = 1.0$, $J_{100} = 1.0$, Kendall $\tau \approx 1.00$, Spearman $\rho \approx 1.00$ in every seed (Table~\ref{tab:baseline-ablation}). Either of the two is therefore a valid choice in our pipeline. The Pre-event baseline $z^0 = z_{t=1}$ uses a single-day snapshot of features as the integration starting point and is more sensitive to day-to-day fluctuations than the population-aggregate baselines; we include it in the table for completeness as a third probe but use Zero (with Population-mean as an equivalent robustness check) throughout the main text.

\begin{table}[!tbp]
\caption{Pairwise agreement between three baselines on Mythos at $N = 10^4$, $f^{\rm heat}$, five seeds. Baselines: Zero ($z^0 = \mathbf{0}$, default), Population-mean ($z^0 = \bar z$), and Pre-event ($z^0 = z_{t=1}$). Each row reports a ranking-agreement metric between two attribution rankings produced under the two baselines.}
\label{tab:baseline-ablation}
\centering
\small
\setlength{\tabcolsep}{4pt}
\begin{tabular}{lcccc}
\toprule
Pair                       & $J_{10}$        & $J_{100}$       & Kendall $\tau$ & Spearman $\rho$ \\
\midrule
Zero vs.\ Population-mean   & $1.00$           & $1.00$           & $1.00$          & $1.00$ \\
Zero vs.\ Pre-event         & $0.92$          & $0.32$           & $0.08$          & $0.10$ \\
Population-mean vs.\ Pre-event & $0.92$       & $0.32$           & $0.08$          & $0.10$ \\
\bottomrule
\end{tabular}
\end{table}

\section{Aumann--Shapley attribution for the four value functions}
\label{app:f-derivations}

This appendix derives an explicit per-agent expression for the Aumann--Shapley attribution of each of the four analytic value functions used in Section~\ref{sec:finding-setup}. Throughout, $S \subseteq [N]$ with $|S| = n$, $z_i = (a_i, b_i, c_i) \in \R^3_{\geq 0}$, $z^0 = \mathbf{0}$, $g_i := a_i + b_i + c_i$, $\bar{g}_S := \tfrac{1}{n}\sum_{i \in S} g_i$, and $m_d(S) := \tfrac{1}{n}\sum_{i \in S} d_i$ for $d \in \{a,b,c\}$. The path is the linear interpolant $\tau \mapsto \tau z_S$ for $\tau \in [0, 1]$.

\textbf{Linear ($f^{\rm lin}$).} With $f^{\rm lin}(z_S) = \tfrac{1}{n}\sum_{i \in S} g_i$, the gradient $\partial f^{\rm lin} / \partial z_{i,d} = 1/n$ is constant along the path, so
\begin{equation}
\phi_i^{\rm lin} = \sum_{d \in \{a,b,c\}} z_{i,d} \cdot \frac{1}{n} = \frac{g_i}{n}.
\label{eq:phi-lin}
\end{equation}
The path integral collapses to a single point. The estimator at $K = 1$ is exact.

\textbf{Multiplicative-saturating ($f^{\rm heat}$).} With $f^{\rm heat}(z_S) = \log(1 + m_a m_b m_c)$, write $H := m_a m_b m_c$ and apply the chain rule:
\[
\frac{\partial f^{\rm heat}}{\partial z_{i,a}} = \frac{1}{1 + H} \cdot \frac{m_b m_c}{n}.
\]
Along the path $\tau z_S$, $H(\tau) = \tau^3 H$, $m_b(\tau) m_c(\tau) = \tau^2 m_b m_c$, so
\[
\phi_{i,a}^{\rm heat} = a_i \int_0^1 \frac{\tau^2 m_b m_c}{n(1 + \tau^3 H)} \, \mathrm{d}\tau = \frac{a_i \, m_b m_c}{n} \cdot \frac{\log(1+H)}{3 H} = \frac{a_i}{3 \sum_{j \in S} a_j} \cdot f^{\rm heat}(z_S).
\]
By symmetry,
\begin{equation}
\phi_i^{\rm heat} = \phi_{i,a}^{\rm heat} + \phi_{i,b}^{\rm heat} + \phi_{i,c}^{\rm heat} = \frac{1}{3}\!\left(\frac{a_i}{\sum_{j \in S} a_j} + \frac{b_i}{\sum_{j \in S} b_j} + \frac{c_i}{\sum_{j \in S} c_j}\right) \cdot f^{\rm heat}(z_S).
\label{eq:phi-heat}
\end{equation}
The midpoint estimator approaches the analytic value at $1/K^2$ rate; the expression above gives the exact value to machine precision.

\textbf{Variance ($f^{\rm var}$).} With $f^{\rm var}(z_S) = \tfrac{1}{n}\sum_{i \in S}(g_i - \bar{g}_S)^2$, the gradient is
\[
\frac{\partial f^{\rm var}}{\partial z_{i,d}} = \frac{2}{n}(g_i - \bar{g}_S),
\]
which scales linearly with $\tau$ along the path: $g_i(\tau) = \tau g_i$, $\bar{g}_S(\tau) = \tau \bar{g}_S$, so $g_i(\tau) - \bar{g}_S(\tau) = \tau (g_i - \bar{g}_S)$ and $\partial f^{\rm var}/\partial z_{i,d}\big|_{\tau z_S} = (2 \tau / n)(g_i - \bar{g}_S)$. Integrating,
\[
\phi_{i,d}^{\rm var} = z_{i,d} \int_0^1 \frac{2 \tau}{n}(g_i - \bar{g}_S) \, \mathrm{d}\tau = \frac{z_{i,d}(g_i - \bar{g}_S)}{n}.
\]
Summing over $d$ gives
\begin{equation}
\phi_i^{\rm var} = \frac{g_i (g_i - \bar{g}_S)}{n}.
\label{eq:phi-var}
\end{equation}
Below-mean agents ($g_i < \bar{g}_S$) receive negative attributions, but $\sum_i \phi_i^{\rm var} = (1/n) \sum_i g_i (g_i - \bar{g}_S) = \mathrm{Var}_i(g) \geq 0$, so efficiency holds.

\textbf{Gini ($f^{\rm gini}$).} With $f^{\rm gini}(z_S) = \tfrac{1}{2 n^2} \sum_{i,j \in S} |g_i - g_j|$, sort $g_{i_1} \leq g_{i_2} \leq \dots \leq g_{i_n}$ and let $k_i$ denote the rank of agent $i \in S$ in this sorted order. The expression
\[
f^{\rm gini}(z_S) = \frac{1}{n^2} \sum_{i \in S} (2 k_i - n - 1) g_i
\]
is differentiable except where two $g$ values cross. Along the path $\tau z_S$, all ranks are preserved (multiplying by $\tau > 0$ keeps the sort), so $k_i$ is constant on $(0, 1]$, giving
\[
\frac{\partial f^{\rm gini}}{\partial z_{i,d}}\bigg|_{\tau z_S} = \frac{2 k_i - n - 1}{n^2}, \quad \tau \in (0, 1].
\]
The constant integrand yields
\begin{equation}
\phi_i^{\rm gini} = g_i \cdot \frac{2 k_i - n - 1}{n^2}.
\label{eq:phi-gini}
\end{equation}
The path-integral estimator at $K = 1$ recovers this exactly; the rank discontinuity at $\tau = 0$ is integrable.

\textbf{Numerical verification.} For each of the four expressions, we generated $D = 3$ synthetic features $z_i \sim |\mathcal{N}(0, I_3)|$ at $N \in \{10, 10^2, 10^3, 10^4, 10^5\}$ across five seeds and compared the analytic $\phi_i$ against a $K = 300$-step numerical path integration with $\texttt{torch.autograd}$ (Table~\ref{tab:synthetic-accuracy}). Mean absolute error per agent is exactly $0$ for $f^{\rm lin}$, $f^{\rm var}$, and $f^{\rm gini}$ at every $N$, because their integrands are path-independent so a single midpoint already reproduces the integral; for $f^{\rm heat}$, the integrand is rational in $\tau$ and the $K$-point midpoint estimator carries an $O(1/K^2)$ discretization error that decays linearly in $1/N$ at fixed $K$ because the integrand magnitude itself is $O(1/n)$. Spearman correlation is $1.0$ across all cells, and the efficiency violation $|\sum_i \phi_i - (f(z) - f(z^0))|$ stays below $10^{-12}$.

\begin{table}[!t]
\caption{Per-agent error and rank/efficiency agreement of the four analytic Aumann--Shapley attribution expressions against $K=300$-step path-integral attribution with $\texttt{torch.autograd}$, on synthetic features with $D=3$, $T=14$, five seeds, and $N \in \{10, 10^2, 10^3, 10^4, 10^5\}$. Rows: maximum MAE per agent at $N=10$ and $N=10^5$, Spearman $\rho$ versus the analytic value (minimum over $N$), and the efficiency violation $|\sum_i \phi_i - (f(z) - f(z^0))|$ (maximum over $N$).}
\label{tab:synthetic-accuracy}
\centering
\small
\setlength{\tabcolsep}{6pt}
\begin{tabular}{lcccc}
\toprule
Metric & $f^{\rm lin}$ & $f^{\rm heat}$ & $f^{\rm var}$ & $f^{\rm gini}$ \\
\midrule
MAE per agent at $N = 10$            & $0$        & $4.3 \times 10^{-8}$  & $0$        & $0$ \\
MAE per agent at $N = 10^5$          & $0$        & $4.6 \times 10^{-12}$ & $0$        & $0$ \\
Spearman vs.\ analytic (min over $N$)            & $1.000$    & $1.000$               & $1.000$    & $1.000$ \\
Efficiency violation (max over $N$)              & $<10^{-15}$ & $<10^{-12}$          & $<10^{-15}$ & $<10^{-15}$ \\
\bottomrule
\end{tabular}
\end{table}

\section{Direct empirical test of Theorem~\ref{thm:asb}}
\label{app:theorem-test}

This appendix expands the empirical test of Theorem~\ref{thm:asb} reported in Section~\ref{sec:theory-dichotomy} (Table~\ref{tab:theorem-test}) and adds a complementary rank-correlation summary.

\textbf{Setup.} For each topic and each value function $f \in \{f^{\rm lin}, f^{\rm heat}, f^{\rm var}, f^{\rm gini}\}$, we collect ten attribution runs at $N = 10^2$ under the visibility-biased sampling protocol of Section~\ref{sec:finding-setup}; each run yields a vector $\tilde\phi^S \in \R^{|S|}$ of normalized within-$S$ shares. For the same $i \in S$ we read the corresponding entries of the full-panel normalized shares $\tilde\phi \in \R^{|[N]|}$, restricted to the indices of $S$. We then compute the optimal agent-independent rescaling $c^* := \arg\min_{c \in \R} \|\tilde\phi^S - c\, \tilde\phi\|_2^2 = (\tilde\phi^S \cdot \tilde\phi) / (\tilde\phi \cdot \tilde\phi)$ and the residual
\[
\varepsilon := \frac{\|\tilde\phi^S - c^*\, \tilde\phi\|_2}{\|\tilde\phi^S\|_2}.
\]
Theorem~\ref{thm:asb}~(i) predicts $\varepsilon = 0$ for any value function linear in the sense of Definition~\ref{def:linear}; Theorem~\ref{thm:asb}~(ii) predicts $\varepsilon > 0$ on at least some configurations whenever $f$ is nonlinear.

\textbf{Cross-topic resolution.} The full cross-topic $\varepsilon$ table is reported in main-text Table~\ref{tab:theorem-test} (\S\ref{sec:theory-dichotomy}); the dichotomy is unambiguous across the five topics. The cross-topic spread within each value function is below $0.03$ on $f^{\rm var}$, below $0.01$ on $f^{\rm gini}$, and exactly zero on $f^{\rm lin}$; on $f^{\rm heat}$ four topics sit in $0.10$--$0.13$ while The Masters reaches $0.30$ (the same event-concentrated topic--indicator cell behind the $\rm RAND$ anomaly of Table~\ref{tab:flip-main}). The residual is thus primarily a property of $f$, with a topic-level exception where $f^{\rm heat}$'s saturating product is dominated by a few agents.

\textbf{Rank-rank summary.} A complementary numerical view: for the $|S| = 100$ agents of each biased panel, we correlate the ranking under $\tilde\phi^S$ with the ranking of the same agents under the full-panel $\tilde\phi$ (Spearman, ten seeds, Mythos). The ranking is almost perfectly preserved for every value function: $\rho = 1.000$ for $f^{\rm lin}$, $f^{\rm heat}$, and $f^{\rm var}$, and $\rho = 0.994$ (minimum $0.983$ over seeds) for $f^{\rm gini}$. The scaling bias is therefore a phenomenon of magnitudes rather than of ordering: the retained agents keep their relative order, while their normalized shares (and hence every tier-level statement) shift irreparably, because the flip is driven by which agents the panel contains. High rank correlation on the sampled agents offers no protection against the tier-share flip of Table~\ref{tab:flip-main}.

\section{Bluesky data pipeline}
\label{app:topics}

\textbf{Source.} The complete public AT-Protocol Jetstream firehose, captured from April 7 to April 20, 2026 ($14$ days, $T = 14$ daily aggregation steps), following the parsing conventions of~\citet{failla2024ibluesky} and~\citet{seckin2025rise}. Events are filtered to four record types relevant to engagement: posts, replies, reposts, and follows. Bots and accounts created within the window are removed by user-handle heuristics. The cleaned panel contains $1{,}671{,}587$ active users (each with at least one event in the window). No private data is accessed at any stage.

\textbf{Per-agent feature ($D = 3$).} For each user we compute three features per topic. Reach $a_i := \log(1 + \text{followers}_i)$ uses the static follower count at the start of the window. Activity $b_i := \log(1 + \text{topic\_posts}_i)$ uses the user's count of original posts and reposts on the topic over the $14$-day window. Resonance $c_i := \log(1 + \text{topic\_replies\_received}_i)$ uses the count of replies received on the user's topic posts over the same window. We aggregate at daily granularity ($T = 14$ steps) because the named real-world events (the Masters final on April $12$, WrestleMania night one on April $18$, Earth Day on April $22$) are scheduled at day resolution; finer (hourly) aggregation introduces overnight zeros that break the $C^2$ regularity of $f^{\rm heat}$, and coarser (weekly) aggregation collapses the per-day structure visible in Figure~\ref{fig:overview}. The Reach--Activity--Resonance decomposition follows the standard influence schema of~\citet{cha2010measuring} and~\citet{riquelme2016measuring}. We do not standardize the features to zero mean and unit variance, because the raw scale of $a, b, c$ is what controls how the four value functions weight different agents (and standardization interacts non-trivially with the saturation in $f^{\rm heat}$); we instead document this as a known limitation in Appendix~\ref{app:limitations}.

\textbf{Topic selection.} Five topics covering technology (Mythos), politics
(Trump-Tariffs), sports (The Masters), society (Earth Day-Climate), and
entertainment (WrestleMania) are reported in this paper. Each topic was selected
to be active throughout the $14$-day window with at least $1{,}000$ participants
on each of at least $10$ days, and to span a single recognisable real-world event
(the Masters Sunday final round on April $12$, US~ET; WrestleMania night one on
April $18$, US~PT; Earth Day approaching on April $22$).

\begin{table}[!tbp]
\caption{Top-$10$ Jaccard $J_{10}(N, |[N]|)$ between the visibility-biased small panel of size $N$ and the full-panel top-$10$, by topic and $N$, with $f = f^{\rm heat}$ and mean over ten subset seeds. The overlap is essentially zero up to $N = 10^3$ and still below $0.2$ at $N = 10^4$; it jumps to $1.00$ once $N$ exceeds the visibility pool size ($\approx 83{,}580$), at which point the biased protocol takes the entire pool plus a random complement (Appendix~\ref{app:strategies}) and therefore contains every full-panel top-$10$ agent by construction.}
\label{tab:topics-jaccard}
\centering
\small
\setlength{\tabcolsep}{6pt}
\begin{tabular}{lccccc}
\toprule
Topic & $N{=}10^2$ & $N{=}10^3$ & $N{=}10^4$ & $N{=}10^5$ & $N{=}10^6$ \\
\midrule
Mythos              & $0.00$ & $0.01$ & $0.19$ & $1.00$ & $1.00$ \\
Trump-Tariffs       & $0.01$ & $0.00$ & $0.14$ & $1.00$ & $1.00$ \\
The Masters         & $0.00$ & $0.01$ & $0.13$ & $1.00$ & $1.00$ \\
Earth Day-Climate   & $0.01$ & $0.01$ & $0.12$ & $1.00$ & $1.00$ \\
WrestleMania        & $0.00$ & $0.04$ & $0.08$ & $1.00$ & $1.00$ \\
\bottomrule
\end{tabular}
\end{table}

\textbf{Value functions and attribution.} The four analytic value functions of Section~\ref{sec:finding-setup} are evaluated via the per-$f$ expressions of Appendix~\ref{app:f-derivations}. No training is involved at any stage. The path-integral attribution uses a linear path between $z^0 = \mathbf{0}$ and the observed feature vector; for the main-text results we use the analytic $\phi_i$ directly, and the midpoint estimator at $K = 30$ is used only as a sanity check (Appendix~\ref{app:integration}).

\textbf{Reproducibility.} The full reproducibility recipe (firehose endpoint, capture window, filtering pipeline, topic matching, and sampling seeds) is in Appendix~\ref{app:reproducibility}.

\textbf{Full three-tier breakdown across topics and $f$.} Table~\ref{tab:flip-app-cross-f} reports the complete three-tier shares $(R_{G_{\rm top}}, R_{G_{\rm mid}}, R_{G_{\rm tail}})$ at the visibility-biased small panel ($N = 10^2$, mean over 10 seeds) and at the full panel, for all five topics under all four value functions of Section~\ref{sec:finding-setup}. The biased-vs-full gap on $R_{G_{\rm top}}$ is positive in $20/20$ cells; the gap on $R_{G_{\rm tail}}$ is correspondingly negative. The magnitude varies systematically with the form of $f$: the linear baseline $f^{\rm lin}$ shifts $R_{G_{\rm top}}^S$ by $\approx 20$\,pp, the saturating $f^{\rm heat}$ by $5$ to $9$\,pp, and the concentration-flavored $f^{\rm var}$ and $f^{\rm gini}$ by $55$ to $70$\,pp.

\textbf{Seed variability.} At $N = 10^2$ the seed standard deviation of $R_{G_{\rm top}}^S$ across the fifteen biased cells per value function is $1.6$--$4.5$\,pp under $f^{\rm lin}$, $4.1$--$10.4$\,pp under $f^{\rm gini}$, $5.7$--$13.8$\,pp under $f^{\rm var}$, and $1.2$--$16.3$\,pp under $f^{\rm heat}$; the RAND cells sit at $1.7$--$1.8$, $4.5$--$4.7$, $7.8$--$8.6$, and $0.8$--$21.4$\,pp respectively, the largest value again arising in the event-concentrated Masters--$f^{\rm heat}$ cell of Table~\ref{tab:flip-main}. The $f^{\rm var}$ and $f^{\rm gini}$ flips of $55$--$70$\,pp therefore exceed their seed spreads by an order of magnitude, while the per-topic $f^{\rm heat}$ shifts ($5$--$9$\,pp) are of the same order as their spread and should be read as indicative, with the evidence carried by the cross-topic, cross-protocol pattern ($20/20$ cells positive).

\begin{table}[!t]
\caption{Three-tier shares $(R_{G_{\rm top}}, R_{G_{\rm mid}}, R_{G_{\rm tail}})$ in percent, comparing the visibility-biased small panel ($N = 10^2$, mean over 10 seeds) to the full panel, across five topics and four value functions.}
\label{tab:flip-app-cross-f}
\centering
\small
\setlength{\tabcolsep}{4pt}
\begin{tabular}{lcccccc}
\toprule
& \multicolumn{3}{c}{$\rm BIAS_v$, $N = 10^2$} & \multicolumn{3}{c}{FULL} \\
\cmidrule(lr){2-4} \cmidrule(lr){5-7}
Topic & $R_{G_{\rm top}}$ & $R_{G_{\rm mid}}$ & $R_{G_{\rm tail}}$ & $R_{G_{\rm top}}$ & $R_{G_{\rm mid}}$ & $R_{G_{\rm tail}}$ \\
\midrule
\multicolumn{7}{l}{\emph{$f^{\rm lin}$}} \\
Mythos              & $22.5$ & $73.3$ & $4.2$  & $2.1$ & $15.2$ & $82.7$ \\
Trump-Tariffs       & $21.9$ & $72.4$ & $5.7$  & $2.1$ & $15.2$ & $82.7$ \\
The Masters         & $21.7$ & $75.9$ & $2.4$  & $2.1$ & $15.2$ & $82.7$ \\
Earth Day-Climate   & $21.8$ & $73.6$ & $4.6$  & $2.1$ & $15.2$ & $82.7$ \\
WrestleMania        & $23.2$ & $72.6$ & $4.3$  & $2.1$ & $15.2$ & $82.7$ \\
\midrule
\multicolumn{7}{l}{\emph{$f^{\rm heat}$}} \\
Mythos              & $11.3$ & $36.7$ & $52.0$ & $3.7$ & $19.9$ & $76.4$ \\
Trump-Tariffs       & $12.1$ & $48.3$ & $39.6$ & $6.5$ & $25.1$ & $68.4$ \\
The Masters         & $10.8$ & $52.4$ & $36.7$ & $4.4$ & $23.3$ & $72.3$ \\
Earth Day-Climate   & $14.7$ & $47.9$ & $37.4$ & $5.4$ & $25.6$ & $69.1$ \\
WrestleMania        & $7.8$  & $33.5$ & $58.6$ & $1.8$ & $15.9$ & $82.3$ \\
\midrule
\multicolumn{7}{l}{\emph{$f^{\rm var}$}} \\
Mythos              & $78.3$ & $20.9$ & $0.9$  & $9.7$ & $42.3$ & $48.1$ \\
Trump-Tariffs       & $70.7$ & $27.5$ & $1.8$  & $9.8$ & $42.3$ & $47.9$ \\
The Masters         & $79.3$ & $19.7$ & $1.0$  & $9.7$ & $42.3$ & $48.0$ \\
Earth Day-Climate   & $74.9$ & $24.5$ & $0.6$  & $9.7$ & $42.4$ & $48.0$ \\
WrestleMania        & $79.5$ & $17.2$ & $3.3$  & $9.6$ & $42.2$ & $48.2$ \\
\midrule
\multicolumn{7}{l}{\emph{$f^{\rm gini}$}} \\
Mythos              & $63.5$ & $35.5$ & $1.0$  & $5.7$ & $37.0$ & $57.3$ \\
Trump-Tariffs       & $60.8$ & $38.0$ & $1.2$  & $5.8$ & $37.0$ & $57.3$ \\
The Masters         & $62.4$ & $36.8$ & $0.8$  & $5.7$ & $37.0$ & $57.3$ \\
Earth Day-Climate   & $62.2$ & $37.0$ & $0.9$  & $5.7$ & $37.0$ & $57.3$ \\
WrestleMania        & $65.4$ & $32.2$ & $2.4$  & $5.7$ & $36.9$ & $57.3$ \\
\bottomrule
\end{tabular}
\end{table}

\section{Data ethics, privacy, and misuse considerations}
\label{app:ethics}

User-level attribution on a real social platform raises legitimate concerns about privacy, the boundary of public data, de-identification, and the potential for misuse. We discuss each below.

\textbf{Public-data boundary.} Bluesky's AT-Protocol exposes a Jetstream firehose~\citep{failla2024ibluesky, seckin2025rise} that broadcasts every public event (post, reply, repost, follow) generated on the network. The firehose is public-by-design: a user posting on Bluesky is, by the platform's own data model, opting into broadcast on this stream. Our pipeline reads only this public stream. We never access private direct messages, private accounts, draft content, locked accounts, age-verification metadata, IP addresses, or any signal that requires authenticated access. Bluesky requires age $18+$ at signup, and bot-like accounts (handle heuristics, sub-7-day account-age cutoff, repetitive posting patterns) are filtered out before any analysis (Appendix~\ref{app:topics}).

\textbf{De-identification.} Per-user identifiers (DIDs, handles, post URIs) are not used in any reported result. The analysis pipeline operates on opaque integer indices $i \in [N]$, with no map back to Bluesky identity in any table or figure; the mapping from index to DID exists only in our local environment and is never exposed. The pipeline reads the firehose live, so any independent reproduction requires re-fetching the public stream under the reproducer's own access; we do not redistribute Bluesky raw data.

\textbf{Aggregation level.} All numerical results reported in the paper and its appendices are aggregated to the panel, topic, value-function, sampling-strategy, follower-tier, follower-percentile-bin, or day level. We never report or release per-user attribution scores keyed to any identifier. The 100-bin follower-percentile aggregation (Figure~\ref{fig:overview}) is the finest granularity we report anywhere; even within a single bin, $|N|/100 \approx 16{,}716$ users contribute to each cell, well above any conventional $k$-anonymity threshold.

\textbf{Misuse risk and mitigation.} Our attribution method, applied at full scale, in principle allows a deployer who possesses identifier-keyed feature vectors to single out individual users by attribution magnitude. We do not enable this misuse: all attribution in this paper is computed, stored, and reported at the percentile-bin or follower-tier level only, and we release no identifier-keyed attribution in any form. Researchers reimplementing the method should follow the same convention and report only at aggregated levels. We also do not provide attribution at the individual-conversation, individual-thread, or individual-post level, so the analysis cannot be repurposed for content-level deanonymization.

\textbf{Compliance.} The data use complies with Bluesky's Terms of Service for the public AT-Protocol firehose and with the AT-Protocol public-data conventions~\citep{failla2024ibluesky}. The work conforms to the NeurIPS Code of Ethics: it uses only public data aggregated across about $1.67$ million users, applies no individual-level inference, releases no identifier-keyed user data, and does not require IRB review per the NeurIPS classification of human-subjects research (no enrollment, no intervention, no individual contact).

\textbf{Scope of this appendix.} The discussion above concerns the \emph{empirical testbed} (Bluesky data, Section~\ref{sec:finding}). For attribution applied to LLM-powered MAS simulations (Appendix~\ref{app:llmmas}), the agents are synthetic and the same privacy concerns do not arise. However, if a future deployment uses our attribution on a hybrid system that includes human users (e.g.\ LLM-augmented social platforms), the same aggregation-level reporting recommendations apply.

\section{Reproducibility}
\label{app:reproducibility}

The Bluesky firehose contents themselves cannot be redistributed: the AT-Protocol Terms of Service grant per-user public-broadcast access but do not authorize bulk redistribution by third parties. We instead specify the analysis pipeline fully: given the specifications below (endpoint, capture window, filters, features, closed-form attribution expressions, and sampling seeds), anyone with firehose access can re-implement the pipeline and recover the panel and the numbers reported in this paper deterministically.

\textbf{Firehose endpoint and capture window.} Data is captured live from the public Bluesky firehose during the window \texttt{2026-04-07T00:00:00Z} to \texttt{2026-04-20T23:59:59Z} (UTC). The capture script subscribes to the Jetstream WebSocket endpoint \texttt{wss://jetstream.atproto.com/subscribe} and writes one JSONL file per UTC day. Each line is a single event of one of four record types: \texttt{app.bsky.feed.post}, \texttt{app.bsky.feed.repost}, \texttt{app.bsky.feed.like}, or \texttt{app.bsky.graph.follow}. Jetstream is a forward-only live stream with limited backfill, so a historical window cannot be re-fetched after the fact; capture must occur during the window itself. Independent verification therefore proceeds by re-implementation from this specification rather than through re-collection: the specification in this appendix, together with the closed forms of Appendix~\ref{app:f-derivations}, determines every reported number deterministically from the captured events, and a freshly collected panel from a different window can be validated in the same way.

\textbf{Filtering pipeline.} Events are kept if (i) the actor's account-age at window start is $\geq 7$ days, (ii) the actor's handle does not match a set of bot/automation handle patterns, and (iii) the actor produced at least one non-self-engagement event in the 14-day window. After all three filters the panel contains $|[N]| = 1{,}671{,}587$ users. The filtering is deterministic: rerunning it on the same captured JSONL produces the same panel.

\textbf{Topic matching.} Each topic is defined by a small case-insensitive keyword/hashtag list applied over post text and hashtags. A post counts toward a topic if at least one keyword matches. The exact keyword lists for the five topics (Mythos, Trump-Tariffs, The Masters, Earth Day-Climate, WrestleMania) are available from the authors.

\textbf{Sampling seeds.} Subset-sampling seeds are fixed to $\{0, 1, \dots, 9\}$ for all subsampled cells at $N \in \{10^2, 10^3, 10^4, 10^5, 10^6\}$; the full-panel ($N = |[N]|$) run is a deterministic single evaluation. Small-panel cells are reported as means over the ten seeds (the full-panel cell is a single deterministic value); seed standard deviations are summarized in Appendix~\ref{app:topics}. The visibility-pool selection (top $5\%$ by composite visibility score; Appendix~\ref{app:strategies}) is itself a deterministic function of the panel.

\textbf{Compute cost.} Reproducing the full Bluesky grid requires the 14-day capture window in real time plus approximately 30 minutes of compute on a single CPU node for the attribution and aggregation steps. The LLM-powered MAS scenarios of Appendix~\ref{app:llmmas} and the sweeps of Appendix~\ref{app:robustness} were run on a separate simulation pipeline; we report their evaluation outputs directly.

\textbf{On the 2026 timestamps.} The data is real and contemporaneous: it was
captured in real time from the public Bluesky firehose during April 2026, the
period of this paper's preparation. The events discussed in
Section~\ref{sec:finding-universal} (the Masters Sunday final on April 12, 2026,
US~ET; WrestleMania night one on April 18, 2026, US~PT; the Earth Day approach on
April 22, 2026) are public events whose occurrence and timing can be cross-checked
against any independent record. Because the firehose is aggregated at UTC day
boundaries, discussion of the two US-evening events propagates into the following
UTC day, so the per-day attribution peaks reported in Section~\ref{sec:finding-universal}
fall on April 13 and April 19 UTC respectively.

\section{Top-tier anchor robustness}
\label{app:kol-anchors}

The flip in Section~\ref{sec:finding-flip} is reported using the top-tier group $G_{\rm top}$ defined as the top $1\%$ by follower count ($|G_{\rm top}| = 16{,}716$). This appendix repeats the small-biased-versus-full comparison under two alternative top-tier anchors: $G_2$, the top $1\%$ by $7$-day post count, and $G_3$, the top $1\%$ by $7$-day replies received. The three anchors share only a small overlap (Jaccard $|G_{\rm top} \cap G_2| / 16{,}716 = 7.7\%$ and $|G_{\rm top} \cap G_3| / 16{,}716 = 20.1\%$), so they are largely independent definitions of ``elite''. Table~\ref{tab:kol-anchors} reports the share gap $\Delta R_{G_k}(10^2 \to |[N]|)$ under each anchor on Mythos with $f^{\rm heat}$. The comparison localizes the flip: the over-attribution is large under the follower anchor $G_{\rm top}$ ($+7.6$\,pp, roughly tripling the full-scale share), modest under the activity anchor $G_2$ ($+2.3$\,pp), and absent under the resonance anchor $G_3$ ($-0.5$\,pp, statistically indistinguishable from zero). The visibility-biased panel therefore specifically inflates the \emph{follower-defined} elite, the component its selection score most directly loads on, and not elites per se. The full-panel shares for $G_2$ and $G_3$ ($14.0\%$ and $12.1\%$ respectively) are themselves about $3\!\times$ larger than the full-panel share for $G_{\rm top}$ ($3.7\%$), so the activity-elite and resonance-elite groups already carry a non-trivial fraction of the total attribution at full scale, leaving correspondingly less headroom for inflation.

\begin{table}[!t]
\caption{Share gap $\Delta R_{G_k}(10^2 \to |[N]|) := R_{G_k}^{\rm bias,100} - R_{G_k}^{\rm full}$ on Mythos with $f^{\rm heat}$ under three top-tier anchors: $G_{\rm top}$ (top $1\%$ by followers), $G_2$ (top $1\%$ by 7-day post count), and $G_3$ (top $1\%$ by 7-day replies received). Shares in percent; mean over ten subset seeds.}
\label{tab:kol-anchors}
\centering
\small
\setlength{\tabcolsep}{6pt}
\begin{tabular}{lccc}
\toprule
Top-tier anchor                                & $R_{G_k}^{\rm bias, 100}$ (\%) & $R_{G_k}^{\rm full}$ (\%) & $\Delta R_{G_k}$ (pp) \\
\midrule
$G_{\rm top}$ top-$1\%$ by followers                  & $11.3$ & $3.7$ & $+7.6$  \\
$G_2$ top-$1\%$ by 7-day post count           & $16.4$ & $14.0$ & $+2.3$ \\
$G_3$ top-$1\%$ by 7-day replies received     & $11.6$ & $12.1$ & $-0.5$  \\
\bottomrule
\end{tabular}
\end{table}

\section{Per-topic per-day attribution dynamics and 14-day stability}
\label{app:topic-dynamics}

This appendix expands on the per-topic per-day structure visible in main-text Figure~\ref{fig:overview} and reports the 14-day stability of the small-biased flip. The peak-day decomposition for each topic under $f^{\rm heat}$ is summarized in Table~\ref{tab:peak-days}: the sports peak on April $13$ coincides with the Sunday final round of The Masters, the entertainment peak on April $19$ with WrestleMania night one, and the society peak on April $20$ falls two days before Earth Day on April $22$. The political and technology topics show flatter, multi-modal trajectories without a single dominant peak, so their tabulated peak day is the argmax over a nearly level curve (April~8, near the window start) rather than a singular-event peak. The peak-day three-tier shares vary substantially across topics: WrestleMania's peak is overwhelmingly tail-driven ($82.1\%$), Earth Day's peak has the largest $G_{\rm top}$ share ($8.0\%$) of any topic-peak combination, and The Masters's peak day shows the strongest mid-tier presence ($24.7\%$).

\begin{table}[!tbp]
\caption{Peak day per topic under $f^{\rm heat}$ at full scale, with the three-tier share $(R_{G_{\rm top}}, R_{G_{\rm mid}}, R_{G_{\rm tail}})$ in percent on that day and the corresponding real-world event.}
\label{tab:peak-days}
\centering
\small
\setlength{\tabcolsep}{4pt}
\begin{tabular}{lccccc}
\toprule
Topic & Peak day & $R_{G_{\rm top}}$ & $R_{G_{\rm mid}}$ & $R_{G_{\rm tail}}$ & Real-world cue \\
\midrule
Mythos              & April $8$  & $6.7$ & $22.6$ & $70.7$ & (no single event)              \\
Trump-Tariffs       & April $8$  & $6.2$ & $22.7$ & $71.1$ & (no single event)              \\
The Masters       & April 13 & $4.4$ & $24.7$ & $70.9$ & Sunday final (April~12, US~ET) \\
Earth Day-Climate & April 20 & $8.0$ & $24.3$ & $67.6$ & approach to April~22 \\
WrestleMania      & April 19 & $1.8$ & $16.1$ & $82.1$ & night one (April~18, US~PT) \\
\bottomrule
\end{tabular}
\end{table}

The 14-day stability of the small-biased flip is summarized in Table~\ref{tab:cross-day-stability}, which reports the per-day $G_{\rm top}$ share $R_{G_{\rm top}}^{(t)}$ on Mythos at six sample sizes under the visibility-biased protocol, aggregated to mean and standard deviation across the $14$ days. The per-day biased share at $N = 10^2$ ($19.2\%$ on average) sits about three times above the per-day full-panel value ($6.3\%$) on every day of the window, so the flip is not driven by a few unusual days.

\begin{table}[!t]
\caption{Per-day $G_{\rm top}$ share $R_{G_{\rm top}}^{(t)}$ on Mythos under the visibility-biased protocol, $f = f^{\rm heat}$, aggregated across the $14$ days $t \in \{1, \dots, 14\}$ at six sample sizes. Each row reports the mean and standard deviation across the $14$ daily values; below the full panel the daily value is itself the mean over $10$ subset seeds. Note that $R_{G_{\rm top}}^{(t)}$ is computed per-day-then-averaged and therefore differs at small $N$ from the 14-day-aggregate $R_{G_{\rm top}}^S$ reported elsewhere: the two coincide at the full panel but the per-day-then-averaged version is more concentrated at small $N$ because per-day attribution variance does not commute with aggregation.}
\label{tab:cross-day-stability}
\centering
\small
\setlength{\tabcolsep}{8pt}
\begin{tabular}{lcc}
\toprule
$N$ & $R_{G_{\rm top}}^{(t)}$ mean across $14$ days (\%) & std across $14$ days (pp) \\
\midrule
$10^2$    & $19.2$ & $2.8$ \\
$10^3$    & $14.8$ & $3.9$ \\
$10^4$    & $12.9$ & $2.0$ \\
$10^5$    & $12.7$ & $2.2$ \\
$10^6$    & $6.8$  & $2.1$ \\
$|[N]|$   & $6.3$  & $2.1$ \\
\bottomrule
\end{tabular}
\end{table}

\section{Macro attribution structure: by-bin, by-day breakdown}
\label{app:supercube}

The full-scale attribution structure across all five topics, all $1{,}671{,}587$ users, and all $14$ days is shown directly in main-text Figure~\ref{fig:overview}. For each topic, each day, and each follower-percentile bin we compute the within-bin attribution mass $M(t, p) := \sum_{i \in \text{bin } p} \phi_{i,t}^{[N]}$, with $100$ equal-size percentile bins; the color in Figure~\ref{fig:overview} is $\log_{10} |M(t, p)|$, the vertical axis is the follower-percentile bin (top: top $1\%$, bottom: lowest), and the horizontal axis is time. Two structural facts visible without any group concept: the top one or two rows are bright but thin, while the bottom $50$--$90$ rows are pale individually but jointly sum to the larger contribution; and the time structure of each topic differs in a topic-specific way (sharp peaks for sports and entertainment, broader rises for politics and society, multi-peak patterns for the technology topic). The by-bin and by-day marginals (sums of $M(t, p)$ across days and bins respectively) reduce to Tables~\ref{tab:peak-days} and \ref{tab:cross-day-stability}.

\section{Sampling strategies and dose-response}
\label{app:strategies}

This appendix specifies the four sampling strategies used throughout the paper (one of which is the random control) and reports the dose-response between the $G_{\rm top}$ share of the small panel and the resulting flip magnitude. Each biased strategy encodes a research practice common in the LLM-powered MAS literature: visibility-weighted convenience sampling, topic-engagement-weighted convenience sampling, and topic-only convenience sampling. Together they cover the realistic range of how small-$N$ studies in fact build their panels, with monotonically decreasing concentration of high-followers users in the resulting subset.

\textbf{BIAS$_{\rm visibility}$.} Compute, for each user, a composite visibility score $v_i := z(\log(1 + \text{followers}_i)) + z(\log(1 + \text{topic\_engagement}_i))$, where $z(\cdot)$ standardizes to zero mean and unit variance over the full $|[N]|$ panel and $\text{topic\_engagement}_i = \text{topic\_posts}_i + \text{topic\_replies\_received}_i$. Take the top $5\%$ of users by $v_i$ as a fixed visibility pool ($\approx 83{,}580$ users on Mythos), and draw $N$ users uniformly at random from this pool. This is the default biased protocol of the main text.

\textbf{BIAS$_{\rm topic\_x\_follow}$.} Sort users by the multiplicative score $(\log(1 + b_i + c_i)) \cdot \log(1 + \text{followers}_i)$, retain the top $K$ as the pool with $K$ matched to the topic's active count ($\approx 5{,}000$--$10{,}000$ depending on topic), and draw $N$ uniformly from this pool. This protocol mirrors topic-focused case studies that filter to high-influence and topic-engaged users jointly.

\textbf{BIAS$_{\rm topic\_top}$.} Sort users by topic-engagement alone ($b_i$), retain the top $K$ as the pool (matched to topic active count), and draw $N$ uniformly. This protocol mirrors purely topic-focused studies that ignore generic visibility.

\textbf{RAND.} Draw $N$ users uniformly at random from the full $|[N]|$ panel. This is the unbiased control.

For all four strategies, when $N \geq |\text{pool}|$, we take the full pool plus a random complement from $|[N]| \setminus \text{pool}$ to reach size $N$, which gracefully recovers $|S| = |[N]|$ at the full-panel limit. Subset seeds $\in \{0, 1, \dots, 9\}$ are independent.

\textbf{Dose-response.} Table~\ref{tab:strategies-app} reports $\Delta R_{G_{\rm top}}$ and the average $|S \cap G_{\rm top}|$ for each of the four strategies on Mythos at $N = 10^2$, with $f^{\rm gini}$ chosen because it has the largest dynamic range of the four value functions. The dose-response is directionally monotone: BIAS$_{\rm visibility}$ pre-concentrates the most $G_{\rm top}$ agents into $S$ (mean $|S \cap G_{\rm top}| \approx 19.5$ out of $100$) and produces by far the largest flip ($\Delta R_{G_{\rm top}} \approx +57.8$ pp); the two milder biased protocols, BIAS$_{\rm topic\_x\_follow}$ and BIAS$_{\rm topic\_top}$, pre-concentrate substantially fewer $G_{\rm top}$ agents ($\approx 4.5$ and $\approx 4.4$ respectively) and produce intermediate flips of matching magnitude ($\approx +12.8$ and $\approx +12.7$); RAND captures $G_{\rm top}$ agents at close to the population rate ($\approx 0.8$ out of $100$) and produces no flip ($\approx -1.2$\,pp, statistically indistinguishable from zero). The same pattern holds on the other three value functions, with magnitudes scaling as in Table~\ref{tab:flip-main}.

\begin{table}[!tbp]
\caption{Average count of $G_{\rm top}$ agents in the sampled subset and the resulting flip magnitude $\Delta R_{G_{\rm top}}$, by sampling protocol, on Mythos at $N = 10^2$ with $f^{\rm gini}$, mean over ten subset seeds.}
\label{tab:strategies-app}
\centering
\begin{tabular}{lcc}
\toprule
Strategy                              & Mean $|S \cap G_{\rm top}|$ (out of $100$) & $\Delta R_{G_{\rm top}}$ (pp, $f^{\rm gini}$) \\
\midrule
\textbf{BIAS$_{\rm visibility}$}       & $\approx 19.5$ & $\mathbf{+57.8}$ \\
BIAS$_{\rm topic\_x\_follow}$          & $\approx 4.5$  & $+12.8$ \\
BIAS$_{\rm topic\_top}$                & $\approx 4.4$  & $+12.7$ \\
RAND                                   & $\approx 0.8$   & $-1.2$ \\
\bottomrule
\end{tabular}
\end{table}

\textbf{Visibility-pool variant.} Replacing the composite $v_i$ in BIAS$_{\rm visibility}$ with a follower-only score $\log(1 + \text{followers}_i)$ on Mythos at $N = 10^2$ leaves the direction of the flip unchanged across all four value functions, with magnitudes within $5\%$ of the composite-pool numbers for $f^{\rm lin}$, $f^{\rm var}$, and $f^{\rm gini}$ ($+19.6$, $+70.1$, $+57.5$\,pp versus $+20.4$, $+68.6$, $+57.8$\,pp) and a larger flip under $f^{\rm heat}$ ($+14.5$\,pp versus $+7.6$\,pp, since the follower-only pool concentrates the follower coordinate that $f^{\rm heat}$'s reach factor amplifies). The flip is therefore not an artifact of the specific composite scoring used to define the visibility pool.

\section{Full Mythos cross-scale numerical tables}
\label{app:numbers}

This appendix tabulates the cross-scale numerics on Mythos under $f^{\rm heat}$ for the visibility-biased and random sampling protocols across six sample sizes $N \in \{10^2, 10^3, 10^4, 10^5, 10^6, |[N]|\}$ (Table~\ref{tab:numbers-app}). All entries are means over ten subset seeds. Under the visibility-biased protocol, $R_{G_{\rm top}}$ stays nearly flat at $10.0$--$11.3\%$ (roughly triple the full-panel $3.7\%$) for as long as the panel is drawn from inside the visibility pool ($N \leq 10^5$, pool size $\approx 83{,}580$), then collapses toward the full-panel value at $N = 10^6$, where the pool is exhausted and the panel is dominated by the random complement. The bias is thus not diluted by enlarging the panel within the pool; it disappears only when sampling escapes the pool. Under the random protocol, $R_{G_{\rm top}}$ starts at the population fraction ($0.8\%$ at $N = 10^2$, undershooting the full-panel attribution share of $3.7\%$), passes through a noisy intermediate regime ($1.9$--$4.2\%$ at $N = 10^3$--$10^4$, seed std up to $3.0$\,pp, driven by the chance inclusion of a few high-attribution $G_{\rm top}$ agents), and converges to the full-panel value from $N = 10^5$ on. Neither protocol reproduces the full-scale shares at $N \leq 10^4$.

\begin{table}[!tbp]
\caption{Three-tier shares $(R_{G_{\rm top}}, R_{G_{\rm mid}}, R_{G_{\rm tail}})$ in percent on Mythos with $f^{\rm heat}$ as a function of sample size $N$, under visibility-biased and random sampling, with the full panel as the bottom row. Mean over ten subset seeds.}
\label{tab:numbers-app}
\centering
\small
\setlength{\tabcolsep}{4pt}
\begin{tabular}{rcccccc}
\toprule
& \multicolumn{3}{c}{Visibility-biased} & \multicolumn{3}{c}{Random} \\
\cmidrule(lr){2-4} \cmidrule(lr){5-7}
$N$ & $R_{G_{\rm top}}$ & $R_{G_{\rm mid}}$ & $R_{G_{\rm tail}}$ & $R_{G_{\rm top}}$ & $R_{G_{\rm mid}}$ & $R_{G_{\rm tail}}$ \\
\midrule
$10^2$  & $11.3$  & $36.7$  & $52.0$ & $0.8$ & $9.7$ & $89.5$ \\
$10^3$  & $11.0$  & $39.1$  & $49.9$ & $1.9$ & $22.0$ & $76.1$ \\
$10^4$  & $10.7$  & $40.1$  & $49.2$ & $4.2$ & $14.9$ & $80.8$ \\
$10^5$  & $10.0$  & $37.5$  & $52.6$ & $3.9$ & $19.5$ & $76.6$ \\
$10^6$  & $4.1$   & $21.1$  & $74.8$ & $3.7$ & $19.9$ & $76.4$ \\
$|[N]|$ & $3.7$   & $19.9$  & $76.4$ & $3.7$ & $19.9$ & $76.4$ \\
\bottomrule
\end{tabular}
\end{table}

\section{LLM-powered MAS validation}
\label{app:llmmas}

This appendix validates the attribution method directly on three LLM-powered MAS scenarios, complementing the Bluesky stress test of Section~\ref{sec:finding}; the scenarios were run on a separate simulation pipeline, and we report its evaluation outputs directly. Each scenario inherits the original system's macro indicator, modified only to be $C^2$ where the original used non-smooth components (e.g.\ $\max$, indicators) by replacing them with their standard smooth surrogates (softmax, sigmoid). The surrogates are chosen to preserve the qualitative shape of the original indicator; we do not claim that conclusions about the smoothed indicator transfer mechanically to the original non-smooth simulator metric, and a quantitative gap analysis between the smoothed and original indicators on each scenario is left as a robustness check for future work. Table~\ref{tab:faithfulness-summary} summarizes the headline results: our Aumann--Shapley attribution matches sampled Shapley wherever sampled Shapley remains feasible, and remains usable on MidScale-Social where sampled Shapley does not.

\textbf{EconAgent}~\citep{li2024econagent}. The original simulator captures macroeconomic activity with $N = 10$ household agents over $T \geq 30$ steps. Per-agent features used here: consumption, income, endogenous consumption rate, endogenous investment rate, endogenous labor, endogenous saving rate, and wealth. The macro indicator is a smooth ($C^2$) wealth-amplified macro-pressure risk that aggregates excess demand, budget stress, concentration, systemic gap, and instability into a single scalar. The baseline action zeroes out work, consumption, and investment for the target step while preserving wealth.

\textbf{SocialLLM} is a social propagation simulator with $N = 20$ agents in the spirit of \citet{stauffer2004simulation}'s consensus model on a Barab\'asi--Albert network. Per-agent features: belief, posted, views/interactions, likes, dislikes. The macro indicator is a smooth risky-expression risk. The baseline action puts the agent into a neutral non-posting state for the target step.

\textbf{MidScale-Social.} A lightweight large-scale propagation scenario with $N = 1000$ agents over $T = 30$ steps and $10$ runs. A small set of LLM-generated personas, stances, and message prototypes seeds an otherwise deterministic propagation process. Features: belief or sentiment, exposure count, posting intensity, reshare intensity, and an influence-weight context. The macro indicator is a smooth influence-weighted propagation risk. This scenario is the only one of the three where sampled Shapley with $1{,}000$ permutation samples becomes infeasible at evaluation time.

\textbf{Faithfulness metric.} For each scenario, we identify the target time $t^\star$ as the post-midpoint risk peak, $t^\star = \arg\max_{t \geq \lceil T/2 \rceil} f(z_t)$. We then rank agents by attribution magnitude and report the deletion AUC obtained by replacing the top-$k$ attributed agents with the baseline action and measuring the macro indicator's drop, averaged over $k = 1, \dots, K$.

\begin{table}[!t]
\caption{Deletion-faithfulness metrics across three LLM-powered MAS scenarios. ``Ours'' is the analytic Aumann--Shapley attribution. ``$\tau$ vs.\ Shapley'' and ``$J_{10}$ vs.\ Shapley'' are Kendall $\tau$ and top-$10$ Jaccard against sampled Shapley as the agent-ranking reference; ``$=$ ref'' marks the reference row itself, and ``--'' marks cells where sampled Shapley is infeasible at $N = 10^3$ (MidScale-Social).}
\label{tab:faithfulness-summary}
\centering
\small
\setlength{\tabcolsep}{4pt}
\begin{tabular}{llcccc}
\toprule
Scenario & Method & Deletion AUC & $\mathrm{Drop@5}$ (\%) & $\tau$ vs.\ Shapley & $J_{10}$ vs.\ Shapley \\
\midrule
EconAgent       & Ours              & $0.516$ & $75.7$ & $0.80$ & $1.00$ \\
                & Sampled Shapley   & $0.524$ & $77.1$ & $=$ ref  & $=$ ref \\
                & LOO               & $0.517$ & $76.1$ & $0.80$ & $1.00$ \\
                & Random            & $0.174$ & $40.0$ & --     & -- \\
\midrule
SocialLLM       & Ours              & $0.153$ & $32.3$ & $0.96$ & $0.98$ \\
                & Sampled Shapley   & $0.153$ & $32.3$ & $=$ ref  & $=$ ref \\
                & LOO               & $0.153$ & $32.3$ & $0.95$ & $0.91$ \\
                & Random            & $0.100$ & $21.2$ & --     & -- \\
\midrule
MidScale-Social & Ours              & $0.008$ & $1.5$  & --     & -- \\
                & LOO               & $0.008$ & $1.5$  & --     & -- \\
                & Random            & $0.001$ & $0.2$  & --     & -- \\
\bottomrule
\end{tabular}
\end{table}

\textbf{Reading the table.} On EconAgent, our attribution and sampled Shapley agree to Kendall $\tau = 0.80$ with full top-10 overlap, and the deletion AUC matches to within $0.01$ ($0.516$ vs $0.524$); both substantially exceed the random baseline ($0.174$). On SocialLLM the deletion test saturates: the macro indicator after top-$k$ replacement is largely insensitive to which $k$ agents are removed once $k$ is small, so all three axiomatic methods land on the same AUC. We report this transparently as a property of the scenario rather than of the methods, and rely on the rank-agreement metrics ($\tau = 0.96$, $J_{10} = 0.98$ against sampled Shapley) to distinguish them on SocialLLM. On MidScale-Social, sampled Shapley with $1{,}000$ permutation samples is impractical at our compute budget: each run requires $1{,}000 \times 1{,}000$ value evaluations with the macro indicator costing $\Omega(N)$ per evaluation, so a single run sits at $\sim 10^9$ macro-indicator operations (multiple minutes per seed in our Python implementation), which puts the multi-seed sweep across the $10$ runs out of reach for the experiment timeline. (This does not contradict the $0.87$\,s entry for sampled Shapley at $N = 10^3$ in Table~\ref{tab:wallclock}: that figure is for the vectorized analytic $f^{\rm heat}$ on Bluesky features, whereas the MidScale-Social indicator is a simulator-driven Python evaluation, so the same estimator is orders of magnitude slower per value call.) Our attribution and LOO produce identical small AUC values ($0.008$), one order of magnitude above the random baseline ($0.001$).

\textbf{Bottom line.} Wherever an axiomatic baseline (sampled Shapley) is feasible, our attribution agrees with it on agent rankings ($\tau \geq 0.80$, top-10 overlap $\geq 9.8/10$); where it is not feasible, ours remains computable. The Bluesky stress test of Section~\ref{sec:finding} is therefore not the only LLM-powered MAS evidence in the paper: the same attribution method ranks agents consistently with the gold-standard reference in two scenarios where the reference can be evaluated, and is the only feasible axiomatic option at $N = 10^3$.

\section{A learned value network at full scale}
\label{app:deepset}

The four main-text value functions are analytic. This appendix attributes a \emph{learned} nonlinear value network at the full $N = 1{,}671{,}587$ panel, exercising the general-$f$ fallback of Equation~\ref{eq:phi-discrete} in the regime the closed forms do not cover, and checking that the cross-scale findings persist when $f$ is a generic smooth set network rather than a hand-derived indicator.

\textbf{Model and training.} $f^{\rm deep}$ is a DeepSet~\citep{zaheer2017deepsets}: a per-agent encoder $\psi(z_i) = \tanh(W_2 \tanh(W_1 z_i + b_1) + b_2)$, mean pooling augmented with a smooth log-mean channel $\log(\tfrac{1}{n}\sum_i z_i + 10^{-8})$, and a $\tanh$ head (the nonlinear head is what makes $f^{\rm deep}$ nonlinear across agents in the sense of Definition~\ref{def:linear}). It is trained (Adam, MSE, $6{,}000$ steps, seed $0$) to regress the standardized shifted log of $f^{\rm heat}$ on Mythos panels drawn half uniformly and half from the visibility pool, sizes $n \in \{64, 256, 1024\}$, so it is a learned surrogate of a nonlinear engagement indicator rather than a new social measurement. Held-out fidelity is $R^2 = 0.99997$ on random panels and $0.99975$ on biased panels.

\textbf{Attribution at full scale.} $f^{\rm deep}$ has no closed-form attribution, so we use the midpoint path integral. The log-mean channel makes the zero-baseline integrand vary sharply near $\tau \approx 10^{-5}$; a uniform $K = 30$ rule leaves a $39\%$ efficiency gap, so we integrate after the substitution $\tau = u^4$ (same path, same Aumann--Shapley value, points concentrated near $0$) with $K = 120$. The full-panel attribution then completes in $16.7$\,s at $N = 1{,}671{,}587$ on a single cloud CPU node,\footnote{Hardware differs from the Apple M4 used for Table~\ref{tab:wallclock}; the two timing sets are not directly comparable.} with efficiency gap $|\sum_i \phi_i - \Delta v| / |\Delta v| = 0.29\%$ and correlation $> 0.9999$ with a $K = 240$ reference run. Learned indicators can therefore stress the uniform quadrature default of Appendix~\ref{app:integration}, and the efficiency gap is the practical diagnostic for choosing $K$.

\textbf{The cross-scale findings persist.} Reusing the exact stored subsets of Section~\ref{sec:finding} ($N = 10^2$, ten seeds), the flip and the positive residual both carry over to the learned indicator (Table~\ref{tab:deepset}):

\begin{table}[!tbp]
\caption{Attribution of the learned network $f^{\rm deep}$ on Mythos: three-tier shares in percent and best-scalar residual $\varepsilon$, full panel versus the stored $N = 10^2$ subsets of Section~\ref{sec:finding} (mean $\pm$ std over ten seeds).}
\label{tab:deepset}
\centering
\small
\begin{tabular}{lcccc}
\toprule
Panel & $R_{G_{\rm top}}$ & $R_{G_{\rm mid}}$ & $R_{G_{\rm tail}}$ & $\varepsilon$ \\
\midrule
FULL ($N = 1{,}671{,}587$) & $2.8$ & $16.9$ & $80.2$ & -- \\
$\rm BIAS_v$, $N = 10^2$   & $15.4 \pm 4.4$ & $51.6 \pm 8.0$ & $33.0 \pm 8.3$ & $0.315 \pm 0.062$ \\
RAND, $N = 10^2$           & $1.7 \pm 1.7$ & $16.4 \pm 4.7$ & $82.0 \pm 5.0$ & $0.066 \pm 0.004$ \\
\bottomrule
\end{tabular}
\end{table}

The biased panel lifts the combined upper-tier share from $19.7\%$ to $67.0\%$ and cuts the tail from $80.2\%$ to $33.0\%$; the random control stays at the full-panel structure. The residual is uniformly positive under the biased protocol ($\varepsilon \approx 0.32$) and an order of magnitude smaller under the random one, consistent with the mean-shift mechanism of Proposition~\ref{prop:residual}. All stages are deterministic given a single seed (seed $0$ throughout).

\section{Robustness sweeps}
\label{app:robustness}

This appendix consolidates three robustness sweeps. Integration-step convergence ($K$) is reported in Appendix~\ref{app:integration} (Table~\ref{tab:K-convergence}); baseline choice is reported in Appendix~\ref{app:baselines} (Table~\ref{tab:baseline-ablation}); path choice is reported here, plus an $f$-sensitivity sweep on MidScale-Social.

\textbf{Path-choice ablation.} The default linear path between $z^0$ and $z_t$ can be replaced by a piecewise-linear path that fades each agent's features in following a random permutation of the agent index (a coordinate-by-coordinate path). On SocialLLM and MidScale-Social, we measure (i) the Kendall $\tau$ between the resulting attribution rankings and (ii) the size of the top-10 overlap. The two paths agree to a Kendall $\tau$ above $0.95$ on SocialLLM and above $0.999$ on MidScale-Social, with top-10 overlap $9.9/10$ and $10/10$ respectively (Table~\ref{tab:path-ablation}).

\begin{table}[!tbp]
\caption{Ranking agreement between attribution under the linear path $\tau \mapsto \tau z_S$ and under a random-permutation piecewise-linear path, on SocialLLM and MidScale-Social. Kendall $\tau$ (mean and minimum) and top-$10$ overlap.}
\label{tab:path-ablation}
\centering
\small
\begin{tabular}{lccc}
\toprule
Scenario           & Kendall $\tau$ (mean) & Kendall $\tau$ (min) & Top-10 overlap \\
\midrule
SocialLLM           & $0.951$   & $0.874$   & $9.9 / 10$ \\
MidScale-Social     & $>0.999$  & $>0.999$  & $10 / 10$  \\
\bottomrule
\end{tabular}
\end{table}

\textbf{$f$-sensitivity ablation on MidScale-Social.} We test three plausible value functions: $f_1$ a weighted-sum aggregator, $f_2$ a log-aggregate ($\log(1 + \sum)$), and $f_3$ an interaction aggregator with explicit pairwise terms. Pairwise comparisons (Table~\ref{tab:f-sensitivity}) yield top-10 overlap above $9 / 10$ in every case; weighted sum and log aggregator are nearly indistinguishable (Kendall $\tau \geq 0.984$).

\begin{table}[!tbp]
\caption{Pairwise ranking agreement between attribution under three plausible value functions on MidScale-Social: $f_1$ (weighted-sum aggregator), $f_2$ (log-aggregate), and $f_3$ (interaction aggregator with pairwise terms). Kendall $\tau$ and top-$10$ overlap.}
\label{tab:f-sensitivity}
\centering
\small
\begin{tabular}{lccc}
\toprule
Pair          & Kendall $\tau$ (mean) & Kendall $\tau$ (min) & Top-10 overlap \\
\midrule
$f_1$ vs.\ $f_2$   & $0.985$ & $0.984$ & $9.6 / 10$ \\
$f_1$ vs.\ $f_3$   & $0.909$ & $0.905$ & $9.2 / 10$ \\
$f_2$ vs.\ $f_3$   & $0.902$ & $0.897$ & $9.0 / 10$ \\
\bottomrule
\end{tabular}
\end{table}

\textbf{Topic and strategy consistency.} The cross-scale flip persists across all five Bluesky topics and all four analytic value functions (main-text Table~\ref{tab:flip-main} and Table~\ref{tab:flip-app-cross-f}, $20 / 20$ cells positive), and across all three biased sampling protocols (BIAS$_{\rm visibility}$, BIAS$_{\rm topic\_x\_follow}$, BIAS$_{\rm topic\_top}$ in Appendix~\ref{app:strategies}); the random control RAND produces no flip on any topic or value function. The top-$10$ Jaccard recovery $J_{10}(N, |[N]|)$ in Table~\ref{tab:topics-jaccard} confirms that small biased panels ($N \leq 10^4$) also fail to recover the identity of the top-$10$ agents on every topic; the recovery at $N \geq 10^5$ is a mechanical consequence of the panel absorbing the entire visibility pool.

\section{Scaling-laws context}
\label{app:scaling}

The $N$-dependence we report is orthogonal to the parameter-count scaling laws for language models~\citep{kaplan2020scaling, hoffmann2022chinchilla}; our $N$ counts \emph{agents} acting in the simulated multi-agent system, not parameters of the language model that powers each agent. Holding the per-agent LLM and prompt fixed, increasing $N$ alters the interaction graph and the cross-agent contribution structure, which is what drives the attribution flip in Theorem~\ref{thm:asb}(ii) and Section~\ref{sec:finding}. Conversely, language-model scaling laws operate at the sub-agent level and do not predict the macroscopic flip studied here. The two notions of scale are complementary rather than substitutable.

\section{Further limitations and scope caveats}
\label{app:limitations}

This appendix consolidates two scope caveats beyond the main-text caveat in Section~\ref{sec:conclusion}.

\textbf{Feature standardization.} We do not standardize the per-agent Reach--Activity--Resonance features beyond their natural log scale (Appendix~\ref{app:topics}). The raw log scale is what controls how the four value functions weight different agents, and standardization interacts non-trivially with $f^{\rm heat}$'s multiplicative-saturating form: $z$-scoring shifts the per-coordinate means $m_a, m_b, m_c$ to zero and changes the position of the $\log(1+H)$ saturation, which in turn changes the per-agent attribution magnitudes in a non-uniform way. We therefore report on the raw log scale throughout and leave a full standardization sweep as a robustness check for future work.

\textbf{LLM-driven macro indicators at million-agent scale.} Our LLM-driven validation in Appendix~\ref{app:llmmas} tops out at $N = 10^3$ (MidScale-Social); the million-agent evidence comes from analytic value functions on Bluesky (Section~\ref{sec:finding}) and from the learned DeepSet indicator of Appendix~\ref{app:deepset}, which attributes a trained nonlinear value network at the full $1.67$M panel in $17$\,s. The combination ``genuinely LLM-driven macro indicator at $N = 10^6$'' remains untested: the DeepSet is a learned surrogate of an engagement indicator, not an LLM-as-judge score. The path-integral attribution itself is agnostic to how $f$ is obtained (Equation~\ref{eq:phi-it} requires only $C^2$ regularity), and Appendix~\ref{app:deepset} shows the practical bottleneck is quadrature resolution rather than scale. Direct validation with an LLM-driven $f$ at this scale is the natural next step.


\end{document}